\def\year{2022}\relax
\documentclass[letterpaper]{article} 
\usepackage{aaai22}  
\usepackage{times}  
\usepackage{helvet}  
\usepackage{courier}  
\usepackage[hyphens]{url}  
\usepackage{graphicx} 
\usepackage{natbib}  
\usepackage{caption} 
\DeclareCaptionStyle{ruled}{labelfont=normalfont,labelsep=colon,strut=off} 
\usepackage{algorithm}
\usepackage{algorithmic}
\usepackage{multibib}
\newcites{Appendix}{Additional References}

\usepackage{newfloat}
\usepackage{listings}
\floatstyle{ruled}
\newfloat{listing}{tb}{lst}{}
\floatname{listing}{Listing}
\nocopyright 

\usepackage{booktabs}
\usepackage{algorithm}
\usepackage{algorithmic}
\usepackage{bm}
\usepackage{amsmath,amssymb,amsthm}

\newtheorem{proposition}{Proposition}

\newtheorem{theorem}{Theorem}

\usepackage{hhline}
\usepackage[font=footnotesize,subrefformat=parens]{subcaption}
\usepackage{arydshln}
\usepackage{multirow}
\usepackage{tikz}
\usetikzlibrary{calc}
\usetikzlibrary{positioning}
\usetikzlibrary{shapes.geometric}
\usetikzlibrary{patterns}
\tikzset{
  vertex/.style={draw, circle,minimum size=0.5cm, inner sep=0pt, anchor=center},
  edge/.style={black,line width=0.1mm},
  vertex-small/.style={fill,inner sep=1.1pt, circle},
  vertex-large/.style={draw, circle,minimum size=0.5cm, inner sep=0pt,anchor=south},
  vertex-large-dead/.style={vertex-large,line width=0.5mm},
  flow/.style={->,line width=0.7mm},
  matched/.style={edge,line width=0.5mm,color={rgb:red,3;green,1;blue,0}},
  cost/.style={rectangle,fill,color=white,draw=black,text=black,inner sep=2pt,anchor=center},
  matching-box/.style={rectangle,minimum size=1.1cm,draw,anchor=north,color={rgb:red,0.1;green,0.1;blue,0.1},line width=0.1mm},
}
\usepackage{macro}
\usepackage[binary-units]{siunitx}

\title{Solving Simultaneous Target Assignment and Path Planning Efficiently\\with Time-Independent Execution}
\author{
  Keisuke Okumura,
  Xavier D\'{e}fago \\
}
\affiliations{
  School of Computing, Tokyo Institute of Technology\\
  Tokyo, Japan\\
  \{okumura.k, defago\}@coord.c.titech.ac.jp
}

\begin{document}

\maketitle
\begin{abstract}
  Real-time planning for a combined problem of target assignment and path planning for multiple agents, also known as the unlabeled version of Multi-Agent Path Finding (MAPF), is crucial for high-level coordination in multi-agent systems, e.g., pattern formation by robot swarms.
  This paper studies two aspects of unlabeled-MAPF:
  (1)~offline scenario:~solving large instances by centralized approaches with small computation time, and
  (2)~online scenario:~executing unlabeled-MAPF despite timing uncertainties of real robots.
  For this purpose, we propose \emph{TSWAP}, a novel sub-optimal complete algorithm, which takes an arbitrary initial target assignment then repeats one-timestep path planning with target swapping.
  TSWAP can adapt to both offline and online scenarios.
  We empirically demonstrate that \emph{Offline TSWAP} is highly scalable; providing near-optimal solutions while reducing runtime by orders of magnitude compared to existing approaches.
  In addition, we present the benefits of \emph{Online TSWAP}, such as delay tolerance, through real-robot demos.
\end{abstract}
\section{Introduction}
Target assignment and path planning for multiple agents, i.e., deciding where to go and how to go, are fundamental problems to achieve high-level coordination in multi-agent systems.
This composite problem has attractive applications such as automated warehouses~\cite{wurman2008coordinating}, robot soccer~\cite{macalpine2015scram}, pattern formation of robot swarms~\cite{turpin2014goal,honig2018trajectory}, a robot display~\cite{alonso2012image}, to name just a few.
These applications typically require real-time planning, i.e., planners have a limited time for deliberation until deadlines.

The problem above is a non-trivial composition of two fundamental problems: 1)~\emph{target assignment} is well-studied~\cite{gerkey2004formal} with well-known efficient algorithms, such as the Hungarian algorithm~\cite{kuhn1955hungarian}; 2)~\emph{path planning}, also known as \emph{Multi-Agent Path Finding (MAPF)}~\cite{stern2019def}, has been extensively studied in recent years.
Given a graph, a set of agents, their initial locations, and their targets, a solution of MAPF maps collision-free paths to agents.
This ``labeled'' MAPF regards targets as being assigned to each agent.
This paper studies the ``unlabeled'' version of MAPF (\emph{unlabeled-MAPF}) which considers agents and targets to be distinct, and hence requires to assign a target to each agent.
In both labeled or unlabeled cases, the main objective is to minimize \emph{makespan}, i.e., the maximum arrival time of agents.

Paradoxically, finding makespan-optimal solutions for unlabeled-MAPF is easier than for MAPF which is known to be NP-hard~\cite{yu2013structure,ma2016multi}. Indeed, unlabeled-MAPF has a polynomial-time optimal algorithm based on a reduction to maximum flow~\cite{yu2013multi};
however, the size of the flow network is quadratic to the size of the original graph, making practical problems in large graphs (e.g., $500\times 500$ grid) still challenging.
Despite its importance, unlabeled-MAPF has received little attention compared to conventional MAPF, for which many scalable sub-optimal solvers have been developed~\cite{surynek2009novel,wang2011mapp,de2013push,okumura2019priority}.

The first objective of this paper is thus to \emph{propose a centralized approach to solve large unlabeled-MAPF instances with sufficiently good quality in small computation time}.
We present \emph{Offline TSWAP}, a sub-optimal complete algorithm.
Specifically, Offline TSWAP uses arbitrary assignment algorithms, then repeats one-timestep path planning with target swapping until all agents have reached targets.

We further extend TSWAP to an \emph{online} version, aiming at \emph{executing unlabeled-MAPF despite timing uncertainties of real robots}; the second objective of this paper.
In practice, plan execution on robots is subject to timing uncertainties (e.g., kinematic constraints, unexpected delays, friction, clock drift).
Even worse, the potential for unexpected interference increases with the number of agents because agents' actions usually depend on each other's;
hence perfect on-time execution is unlikely to be expected.

To overcome this problem, we propose \emph{Online TSWAP}, an online version of TSWAP based on the concept of time-independent planning~\cite{okumura2021time}. In other words, it abandons all timing assumptions (e.g., synchronization, traveling time, rotation time, delay probabilities) and regards the whole system as a transition system that changes its configuration according to atomic actions of agents.
Regardless of movement timings, TSWAP ensures that all targets are eventually reached.

Our main contribution is proposing TSWAP to solve or execute unlabeled-MAPF, specifically;
(1)~\emph{offline scenario:}~we propose a novel algorithm and empirically demonstrate that TSWAP is scalable and can yield near makespan-optimal solutions while reducing runtime by orders of magnitude in most cases when compared to the polynomial-time optimal algorithm~\cite{yu2013multi}.
Furthermore, TSWAP also yields good solutions with respect to \emph{sum-of-costs}, another commonly used metric in MAPF studies.
(2)~\emph{online scenario:}~we formulate an online time-independent problem and propose a complete algorithm.
We show the benefits of TSWAP, such as time independence and delay tolerance, through real-robot demos.
Incidentally, (3) we also present efficient assignment algorithms with lazy evaluation of distances, assuming to use with TSWAP.

The paper is structured as follows.
Section~\ref{sec:related-work} summarizes related work about unlabeled-MAPF and time-independent execution methods.
Section~\ref{sec:problem-definition} formalizes offline and online time-independent problems of unlabeled-MAPF.
Section~\ref{sec:offline-tswap} presents Offline TSWAP and its theoretical analysis.
Section~\ref{sec:target-assignment} presents assignment algorithms with lazy evaluation.
Section~\ref{sec:evaluation} presents empirical results of offline planning.
Section~\ref{sec:online-tswap} presents Online TSWAP.
Section~\ref{sec:demo} presents robot demos of online planning.
Section~\ref{sec:conclusion} concludes the paper.
The technical appendix, code, and video are available on \url{https://kei18.github.io/tswap}.

\section{Related Work}
\label{sec:related-work}
\subsection{Target Assignment and Path Planning}
The unlabeled-MAPF problem, also known as \emph{anonymous MAPF}, consists of two sub-problems:
(1)~target assignment, more generally, task allocation, and (2)~path planning.
The \emph{multi-robot task allocation} problems are a mature field~\cite{gerkey2004formal}.
Path planning for multiple agents, embodied as MAPF, has been actively studied in recent years~\cite{stern2019def}.
We focus on reviews of related studies covering both aspects.

Unlike conventional MAPF, unlabeled-MAPF is always solvable~\cite{kornhauser1984coordinating,yu2013multi,adler2015efficient,ma2016multi}.
Among them, TSWAP relates to the analysis presented by~\citet{yu2013multi} because both approaches use target swapping.
Their analysis relies on optimal linear assignment whereas TSWAP works for any assignments.
They also showed that unlabeled-MAPF has a \emph{Pareto optimal structure} for makespan and sum-of-costs metrics (summation of traveling time of each agent; see the next section), i.e., there is an instance for which it is impossible to optimize both metrics simultaneously.
Furthermore, they present a polynomial-time makespan-optimal algorithm, in contrast with conventional MAPF being known to be NP-hard~\cite{yu2013structure,ma2016multi}.

The \emph{combined target assignment and path finding (TAPF)} problem~\cite{ma2016optimal}, also called as \emph{colored MAPF}~\cite{bartak2021classical}, generalizes both MAPF and unlabeled-MAPF by partitioning the agents into teams.
The paper proposes a makespan-optimal algorithm for TAPF that combines Conflict-based Search (CBS)~\cite{sharon2015conflict}, a popular optimal MAPF algorithm, with an optimal algorithm for unlabeled-MAPF.
\citet{honig2018conflict} studied a sum-of-costs optimal algorithm for TAPF by extending CBS.
They also proposed a bounded sub-optimal algorithm, called ECBS-TA; we compare TSWAP with ECBS-TA in the experiment.
There is a study~\cite{wagner2012subdimensional} using another optimal MAPF algorithm~\cite{wagner2015subdimensional} to solve the joint problem of target assignment and path planning.

The \emph{multi-agent pickup and delivery (MAPD)} problem~\cite{ma2017lifelong}, motivated by applications in automated warehouses~\cite{wurman2008coordinating}, aims at making agents convey packages and has to solve target assignment and path planning jointly.
Many approaches to MAPD have been proposed, e.g., \cite{ma2017lifelong,liu2019task,okumura2019priority}.
Although MAPD is a problem different from unlabeled-MAPF, TSWAP is similar to an MAPD algorithm TPTS~\cite{ma2017lifelong} in the sense that both algorithms swap assigned targets adaptively.
One difference though is that, unlike TSWAP, TPTS sets additional conditions about start and target locations.

MAPF is a kind of \emph{pebble motion} problem, in which objects are moved on a graph one-at-a-time, like a sliding tile puzzle.
The unlabeled version of pebble motion has also been studied~\cite{kornhauser1984coordinating,cualinescu2008reconfigurations,goraly2010multi}.
However, in unlabeled-MAPF, agents can move simultaneously;
different from those studies, TSWAP explicitly assumes this fact, resulting in practical outcomes.

Pattern formation of multiple agents~\cite{oh2015survey} is one of the motivating examples of unlabeled-MAPF.
Various approaches have been studied, e.g.,~\cite{alonso2011multi,wang2020shape}.
We highlight two studies closely related to ours as follows.
SCRAM~\cite{macalpine2015scram} is a target assignment algorithm considering collisions and works only in open space without obstacles; hence its applications are limited.
The assignment algorithm in this paper (Alg.~\ref{algo:target-assignment}) uses a scheme similar to SCRAM but differs in its use of lazy evaluation.
\citet{turpin2014goal} proposed a method that first solves the lexicographic bottleneck assignment~\cite{burkard1991lexicographic} then plans trajectories on graphs.
To avoid collisions, the method uses the delay offset about when agents start moving, resulting in a longer makespan.
TSWAP avoids using such offsets by swapping targets on demand.

\subsection{Execution without Timing Assumptions}
\citet{ma2017multi} studied robust execution policies using offline MAPF plans as input, but assuming that agents might be delayed during the execution of timed schedules.
The proposed Minimum Communication Policies (MCPs) make agents preserve two types of temporal dependencies: internal events within one agent and order relation of visiting a node.
Regardless of delays, MCPs make all agents reach their destinations without conflicts.
We later use MCPs for robot demos as a comparison of Online TSWAP.

\citet{okumura2021time} studied \emph{time-independent planning} to execute MAPF by modeling the whole system as a transition system that changes configurations according to atomic actions of agents.
The online problem defined in this paper can be regarded as an unlabeled version of their model for conventional MAPF.

\section{Problem Definition and Terminologies}
\label{sec:problem-definition}

\paragraph{Unlabeled-MAPF Instance}
A problem instance of \emph{unlabeled-MAPF} is defined by a connected undirected graph $G = (V, E)$, a set of agents $A = \{a_1, \ldots, a_n\}$, a set of distinct initial locations $\mathcal{S} = \{s_1, \ldots, s_n \}$ and distinct target locations $\mathcal{T} = \{g_1, \ldots, g_m\}$, where $|\mathcal{T}| \leq |A|$.

\paragraph{Offline Problem}
\label{subsec:problem-offline}
Given an unlabeled-MAPF instance,
let $\loc{i}{t} \in V$ denote the location of an agent $a_i$ at discrete time~$t \in \mathbb{N}$.
At each timestep~$t$, $a_i$ can move to an adjacent node, or can stay at its current location, i.e., $\loc{i}{t+1} \in \neigh{\loc{i}{t}} \cup \{ \loc{i}{t} \}$, where \neigh{v} is the set of nodes adjacent to $v \in V$.
Agents must avoid two types of conflicts:
(1)~\emph{vertex conflict}: $\loc{i}{t} \neq \loc{j}{t}$, and,
(2)~\emph{swap conflict}: $\loc{i}{t} \neq \loc{j}{t+1} \lor \loc{i}{t+1} \neq \loc{j}{t}$.
A \emph{solution} is a set of paths $\{ \path{1}, \ldots, \path{n} \}$ such that a subset of agents occupies all targets at a certain timestep $T$.
More precisely, assign a path $\path{i} = (\loc{i}{0}, \loc{i}{1}, \dots, \loc{i}{T})$ to each agent such that $\loc{i}{0} = s_i$ and there exists an agent $a_j$ with $\loc{j}{T} = g_k$ for all $g_k \in \mathcal{T}$.

We consider four metrics to rate solutions:
\begin{itemize}
  \item \emph{makespan}:
  the first timestep when all targets are occupied, i.e., $T$.
  \item \emph{sum-of-costs}: $\sum_{i}T_i$ where $T_i$ is the minimum timestep such that $\loc{i}{T_i}=\loc{i}{T_i+1}=\ldots=\loc{i}{T}$.
  \item \emph{maximum-moves}: the maximum of how many times each agent moves to adjacent nodes.
  \item \emph{sum-of-moves}: the summation of moves of each agent.
\end{itemize}

\paragraph{Online Time-Independent Problem}
\label{subsec:problem-online}
An \emph{execution schedule} is defined by infinite sequence $\exec = (a_i, a_j, a_k, \ldots)$ defining the order in which each agent is \emph{activated} and can move one step.

Given an unlabeled-MAPF instance, a situation where all agents are at their initial locations, and an execution schedule \exec, an agent~$a_i$ can move to an adjacent node if (1)~it is $a_i$'s turn in \exec and (2)~the node is unoccupied by others.
\exec is called \emph{fair} when all agents appear infinitely often in \exec.
\emph{Termination} is a configuration where all targets are occupied by a subset of agents simultaneously.
An algorithm is called \emph{complete} when termination is achieved within a finite number of activations for any fair execution schedules.

Given an execution schedule, we rate the efficiency of agents' behaviors according to two metrics: \emph{maximum-moves} and \emph{sum-of-moves}.
Their definitions are the same as for the offline problem.

\paragraph{Remarks for Online Problem}
Since any complete algorithms must deal with any fair schedules, they inherently assume timing uncertainties.
For simplicity, we assume that at most one agent is activated at any time, hence the execution is determined by a sequence over the agents.
There is no loss of generality as long as an agent can atomically reserve its next node before each move.
Note that we do not formally define sum-of-costs and makespan for the online problem since they should be measured according to actual time.

\paragraph{Other Assumptions and Notations}
For simplicity, we assume $|\mathcal{T}| = |A|$ unless explicitly mentioned.
We denote the diameter of $G$ by $\diam(G)$, and its maximum degree by $\Delta(G)$.
Let \dist{u}{v} denote the shortest path length from $u \in V$ to $v \in V$.
We assume the existence of admissible heuristics $h(u, v)$ for computing the shortest path length in constant time, i.e., $h(u, v) \leq \dist{u}{v}$, e.g., the Manhattan distance.
This paper uses a simplified notation of the asymptotic complexity like $O(V)$ rather than $O(|V|)$.

\section{Offline TSWAP}
\label{sec:offline-tswap}
This section presents \emph{Offline TSWAP}, a sub-optimal path planning algorithm for the offline problem, which is complete for any initial target assignments.
Here, an \emph{assignment} is a set of pairs $s \in \mathcal{S}$ and $g \in \mathcal{T}$ such that all agents have a distinct target.
Most proofs are deferred to the Appendix.

\subsection{Algorithm Description}
\label{subsec:path-planning}
\begin{algorithm}[tb]
  \caption{\textbf{Offline TSWAP}}
  \label{algo:offline}
  {\small
  \begin{algorithmic}[1]
    \item[\textbf{input}:~unlabeled-MAPF instance]
    \item[\textbf{output}:~plan $\paths$]
      \STATE get an initial assignment $\mathcal{M}$: a set of pairs $s \in \mathcal{S}$ and $g \in \mathcal{T}$
      \label{algo:offline:assign}
    \STATE $a_i.v, a_i.g \leftarrow (s_i, g) \in \mathcal{M}$~:~for each agent $a_i \in A$
    \label{algo:offline:init}
    \STATE $t \leftarrow 0$ \COMMENT{timestep}
    \WHILE{$\exists a \in A, a.v \neq a.g$}
    \label{algo:offline:loop-start}
    \FOR{$a \in A$}
    \label{algo:offline:start-for}
    \IFSINGLE{$a.v = a.g$}{\textbf{continue}}
    \label{algo:offline:at-goal}
    \STATE $u \leftarrow \nextnode{a.v}{a.g}$
    \label{algo:offline:next}
    \IF{$\exists b \in A~\text{s.t.}~b.v = u$}
    \label{algo:offline:if-occupied}
    \IF{$u = b.g$}
    \STATE swap targets of $a$, $b$; $a.g \leftarrow b.g$, $b.g \leftarrow a.g$
    \label{algo:offline:swap}
    \ELSIF{detect deadlock for $A^\prime\subseteq A \land a \in A^\prime$}
    \label{algo:offline:deadlock-detection}
    \STATE rotate targets of $A^\prime$
    \label{algo:offline:rotation}
    \ENDIF
    \ELSE
    \label{algo:offline:if-occupied-end}
    \STATE $a.v \leftarrow u$
    \label{algo:offline:move}
    \ENDIF
    \label{algo:offline:end-main}
    \ENDFOR
    \label{algo:offline:end-for}
    \STATE $t \leftarrow t + 1$
    \STATE $\loc{i}{t} \leftarrow a_i.v$~:~for each agent $a_i \in A$
    \ENDWHILE
    \label{algo:offline:loop-end}
  \end{algorithmic}
  }
\end{algorithm}
TSWAP assumes that an initial assignment is given externally, then mainly determines how to go but not only.
This is because the initial assignment is potentially an unsolvable MAPF instance (e.g., see $t=0$ at path planning in Fig.~\ref{fig:tswap});
we cannot apply MAPF solvers directly to design a complete unlabeled-MAPF algorithm for arbitrary assignments.
The algorithm must consider swapping targets as necessary.

Algorithm~\ref{algo:offline} generates a solution \paths by moving agents incrementally towards their targets following the shortest paths, using the following function;
\begin{align*}
  \nextnode{u}{w} \defeq \argmin_{v \in \neigh{u} \cup \{ u \} }\dist{v}{w}
\end{align*}
$\mathsf{nextNode}$ is assumed to be deterministic; tie-break between nodes having the same scores is done deterministically.

Each agent~$a$ has two variables: $a.v$ is the current location, and $a.g$ is the current target.
They are initialized by the initial assignment [Lines~\ref{algo:offline:assign}--\ref{algo:offline:init}].
After that and until all agents reach their targets, one-timestep planning is repeated as follows [Lines~\ref{algo:offline:loop-start}--\ref{algo:offline:loop-end}].
If $a$ is on its target $a.g$, it stays there ($a.v=a.g$).
Otherwise, $a$ attempts a move to the nearest neighbor of $a.v$ towards $a.g$, call it $u$ [Line~\ref{algo:offline:next}].
When $u$ is occupied by another agent, the algorithm either performs target swapping or deadlock resolution~[Lines~\ref{algo:offline:if-occupied}--\ref{algo:offline:if-occupied-end}].
Here, a deadlock is defined as follows.
A set of agents $A^\prime=(a_{i1}, a_{i2}, a_{i3}, \ldots, a_{ij})$ is in a \emph{deadlock} when $\nextnode{a_{i1}.v}{a_{i1}.g} = a_{i2}.v \land \nextnode{a_{i2}.v}{a_{i2}.g} = a_{i3}.v \land \ldots \land \nextnode{a_{ij}.v}{a_{ij}.g} = a_{i1}.v$.
When detecting a deadlock for $A^\prime$, the algorithm ``rotates'' targets;
$a_{i1}.g \leftarrow a_{ij}.g, a_{i2}.g \leftarrow a_{i1}.g, a_{i3}.g \leftarrow a_{i2}.g, \ldots$ [Line~\ref{algo:offline:rotation}].
The detection incrementally checks whether the next location of each agent is occupied by another agent and concurrently checks the existence of a loop.

Figure~\ref{fig:tswap} shows an example of TSWAP, together with target assignment by Alg.~\ref{algo:target-assignment}$^\dagger$ introduced in Sec.~\ref{sec:target-assignment}.

\begin{figure}[t]
  \centering
  \begin{tikzpicture}
    \newcommand{\betweenvertex}{0.7 cm}
    \scriptsize
    %
    {
      \node[vertex,label=below:{$g_1$}](v1) at (0, 0) {};
      \node[vertex,right=\betweenvertex of v1.center](v2) {};
      \node[vertex,right=\betweenvertex of v2.center,label=above:{$s_1$}](v3) {};
      \node[vertex,right=\betweenvertex of v3.center,,label=above:{$s_2$}](v4) {};
      \node[vertex,right=\betweenvertex of v4.center,label=below:{$g_2$},label=above:{$s_3$}](v5) {};
      \node[vertex,right=\betweenvertex of v5.center,label=below:{$g_3$}](v6) {};
      \foreach \u / \v in {v1/v2, v2/v3, v3/v4, v4/v5, v5/v6} \draw[edge] (\u) -- (\v);
    }
    %
    {
      \node[below=0.3cm of v3] () {Target Assignment (by Alg.~\ref{algo:target-assignment}$^\dagger$)};
      \node[below=3.5cm of v3] () {Path Planning (by Alg.~\ref{algo:offline})};
    }
    %
    {
      \newcommand{\betweenvertexm}{0.25 cm}
      \newcommand{\betweenmatching}{1.2 cm}
      \node[vertex-small,label=above:{\scriptsize $s_1$}](s1_1) at (-0.8, -1.5) {};
      \node[vertex-small,right=\betweenvertexm of s1_1.center,label=above:{\scriptsize $s_2$}](s2_1) {};
      \node[vertex-small,right=\betweenvertexm of s2_1.center,label=above:{\scriptsize $s_3$}](s3_1) {};
      \node[vertex-small,below=\betweenvertexm of s1_1.center,label=below:{\scriptsize $g_1$}](g1_1) {};
      \node[vertex-small,below=\betweenvertexm of s2_1.center,label=below:{\scriptsize $g_2$}](g2_1) {};
      \node[vertex-small,below=\betweenvertexm of s3_1.center,label=below:{\scriptsize $g_3$}](g3_1) {};
      \foreach \u / \v in {s3_1/g2_1} \draw[matched] (\u) -- (\v);
      \node[matching-box,above=0.39cm of s2_1.center,anchor=north](box1) {};
      \node[vertex-small,right=\betweenmatching of s1_1](s1_2) {};
      \node[vertex-small,right=\betweenvertexm of s1_2.center](s2_2) {};
      \node[vertex-small,right=\betweenvertexm of s2_2.center](s3_2) {};
      \node[vertex-small,below=\betweenvertexm of s1_2.center](g1_2) {};
      \node[vertex-small,below=\betweenvertexm of s2_2.center](g2_2) {};
      \node[vertex-small,below=\betweenvertexm of s3_2.center](g3_2) {};
      \foreach \u / \v in {s2_2/g2_2} \draw[edge] (\u) -- (\v);
      \foreach \u / \v in {s3_2/g2_2} \draw[matched] (\u) -- (\v);
      \node[matching-box,above=0.39cm of s2_2.center,anchor=north](box2) {};
      \draw[->,color={rgb:red,0.1;green,0.1;blue,0.1}](box1) -- (box2);
      \node[vertex-small,below=\betweenmatching of s1_1](s1_3) {};
      \node[vertex-small,right=\betweenvertexm of s1_3.center](s2_3) {};
      \node[vertex-small,right=\betweenvertexm of s2_3.center](s3_3) {};
      \node[vertex-small,below=\betweenvertexm of s1_3.center](g1_3) {};
      \node[vertex-small,below=\betweenvertexm of s2_3.center](g2_3) {};
      \node[vertex-small,below=\betweenvertexm of s3_3.center](g3_3) {};
      \foreach \u / \v in {s3_3/g2_3} \draw[edge] (\u) -- (\v);
      \foreach \u / \v in {s2_3/g2_3,s3_3/g3_3} \draw[matched] (\u) -- (\v);
      \node[matching-box,above=0.39cm of s2_3.center,anchor=north](box3) {};
      \draw[->,color={rgb:red,0.1;green,0.1;blue,0.1}](box2) -- (box3);
      \node[vertex-small,right=\betweenmatching of s1_3](s1_4) {};
      \node[vertex-small,right=\betweenvertexm of s1_4.center](s2_4) {};
      \node[vertex-small,right=\betweenvertexm of s2_4.center](s3_4) {};
      \node[vertex-small,below=\betweenvertexm of s1_4.center](g1_4) {};
      \node[vertex-small,below=\betweenvertexm of s2_4.center](g2_4) {};
      \node[vertex-small,below=\betweenvertexm of s3_4.center](g3_4) {};
      \foreach \u / \v in {s3_4/g2_4} \draw[edge] (\u) -- (\v);
      \foreach \u / \v in {s2_4/g2_4,s3_4/g3_4,s1_4/g1_4} \draw[matched] (\u) -- (\v);
      \node[matching-box,above=0.39cm of s2_4.center,anchor=north](box4) {};
      \draw[->,color={rgb:red,0.1;green,0.1;blue,0.1}](box3) -- (box4);
    }
    %
    {
      \newcommand{\betweenvertexm}{1.2 cm}
      \node[vertex-small,label=above:{\scriptsize $s_1$}](s1) at (2.5, -1.7) {};
      \node[vertex-small,right=\betweenvertexm of s1.center,label=above:{\scriptsize $s_2$}](s2) {};
      \node[vertex-small,right=\betweenvertexm of s2.center,label=above:{\scriptsize $s_3$}](s3) {};
      \node[vertex-small,below=\betweenvertexm of s1.center,label=below:{$g_1$}](g1) {};
      \node[vertex-small,below=\betweenvertexm of s2.center,label=below:{$g_2$}](g2) {};
      \node[vertex-small,below=\betweenvertexm of s3.center,label=below:{$g_3$}](g3) {};
      \foreach \u / \v in {s1/g2,s2/g2,s3/g3} \draw[edge] (\u) -- (\v);
      \foreach \u / \v in {s1/g1,s2/g3,s3/g2} \draw[matched] (\u) -- (\v);
      \node[matching-box,above=0.5cm of s2.center,anchor=north,minimum height=2.2cm,minimum width=3cm](box-last) {};
      \draw[->,color={rgb:red,0.1;green,0.1;blue,0.1}](box4) -- (box-last.west);
      \node[cost,below=0.3cm of s1](s1g1) {\tiny $2$};
      \node[cost,below right=0.475cm of s1](s1g2) {\tiny $2$};
      \node[cost,below=0.3cm of s2](s1g1) {\tiny $1$};
      \node[cost,below=0.3cm of s3](s3g3) {\tiny $1$};
      \node[cost,below right=0.18cm of s2](s2g3) {\tiny $2$};
      \node[cost,below left=0.18cm of s3](s3g2) {\tiny $0$};
    }
    {
      \node[vertex](v1) at (0, -4.6) {};
      \node[vertex,right=\betweenvertex of v1.center](v2) {};
      \node[vertex,right=\betweenvertex of v2.center](v3) {$a_1$};
      \node[vertex,right=\betweenvertex of v3.center](v4) {$a_2$};
      \node[vertex,right=\betweenvertex of v4.center](v5) {$a_3$};
      \node[vertex,right=\betweenvertex of v5.center](v6) {};
      \foreach \u / \v in {v1/v2, v2/v3, v3/v4, v4/v5, v5/v6} \draw[edge] (\u) -- (\v);
      \node[left=0.2cm of v1]() {$t=0$};
      \draw[->] (v3) to[out=135,in=45] (v1);
      \draw[->] (v4) to[out=45,in=135] (v6);
      \coordinate[above=0.2cm of v5.north](tmp1);
      \draw[-] (v5.north) to[out=150,in=180] (tmp1);
      \draw[->] (tmp1) to[out=0,in=60] (v5.north);
    }
    {
      \node[vertex,below=0.7cm of v1.center](v1) {};
      \node[vertex,right=\betweenvertex of v1.center](v2) {$a_1$};
      \node[vertex,right=\betweenvertex of v2.center](v3) {};
      \node[vertex,right=\betweenvertex of v3.center](v4) {$a_2$};
      \node[vertex,right=\betweenvertex of v4.center](v5) {$a_3$};
      \node[vertex,right=\betweenvertex of v5.center](v6) {};
      \foreach \u / \v in {v1/v2, v2/v3, v3/v4, v4/v5, v5/v6} \draw[edge] (\u) -- (\v);
      \node[left=0.2cm of v1]() {$t=1$};
      \draw[->] (v2) to[out=135,in=45] (v1);
      \draw[->] (v4) to[out=45,in=135] (v5);
      \draw[->] (v5) to[out=45,in=135] (v6);
    }
    {
      \node[vertex,below=0.7cm of v1.center](v1) {$a_1$};
      \node[vertex,right=\betweenvertex of v1.center](v2) {};
      \node[vertex,right=\betweenvertex of v2.center](v3) {};
      \node[vertex,right=\betweenvertex of v3.center](v4) {};
      \node[vertex,right=\betweenvertex of v4.center](v5) {$a_2$};
      \node[vertex,right=\betweenvertex of v5.center](v6) {$a_3$};
      \foreach \u / \v in {v1/v2, v2/v3, v3/v4, v4/v5, v5/v6} \draw[edge] (\u) -- (\v);
      \node[left=0.2cm of v1]() {$t=2$};
      \coordinate[above=0.2cm of v5.north](tmp1);
      \draw[-] (v5.north) to[out=150,in=180] (tmp1);
      \draw[->] (tmp1) to[out=0,in=60] (v5.north);
      \coordinate[above=0.2cm of v6.north](tmp2);
      \draw[-] (v6.north) to[out=150,in=180] (tmp2);
      \draw[->] (tmp2) to[out=0,in=60] (v6.north);
      \coordinate[above=0.2cm of v1.north](tmp3);
      \draw[-] (v1.north) to[out=150,in=180] (tmp3);
      \draw[->] (tmp3) to[out=0,in=60] (v1.north);
    }
  \end{tikzpicture}
  \caption{\textbf{Example of Offline TSWAP with Alg.~\ref{algo:target-assignment}$^\dagger$ as an assignment algorithm.}
    An unlabeled-MAPF instance is shown at the top.
    The target assignment is illustrated in the middle, using a bipartite graph $\mathcal{B}$.
    The assignment $\mathcal{M}$ is denoted by red lines.
    The left corresponds to finding the bottleneck cost [Lines~\ref{algo:ta:setup}--\ref{algo:ta:end-bap}], i.e., incrementally adding pairs of an initial location and a target then updating the matching.
    The right part, with annotations of costs, corresponds to solving the minimum cost maximum matching problem [Line~\ref{algo:ta:mincost-matching}].
    Two edges are added from the last situation due to line~\ref{algo:ta:additional-edge}.
    Offline TSWAP (Alg.~\ref{algo:offline}) is illustrated at the bottom.
    $a_3$, $a_2$, and $a_1$ repeat one-step planning.
    Current locations of agents, i.e., $a.v$, are shown within nodes.
    Arrows represent the targets, i.e., $a.g$.
    A target swapping happens between $a_2$ and $a_3$ at $t=0$ [Line~\ref{algo:offline:swap}].
    Note that we artificially assign $s_2$ to $g_3$ to show the example of target swapping.
  }
  \label{fig:tswap}
\end{figure}
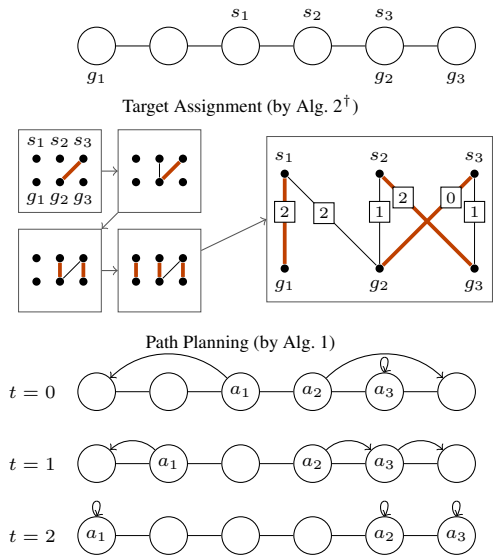

\subsection{Theoretical Analysis}
\label{subsec:offline:analysis}
\begin{theorem}
  Offline TSWAP (Algorithm~\ref{algo:offline}) is complete for the offline problem.
  \label{thrm:offline}
\end{theorem}
\begin{proof}
  Let $\Pi(u, u^\prime) \subset V$ be a set of nodes in a shortest path from $u \in V$ to $u^\prime \in V$, identified by $\mathsf{nextNode}$, except for $u$ and $u^\prime$.
  Consider the following potential function in Alg.~\ref{algo:offline}.
  \begin{align*}
    \phi = \sum_{a \in A}\Bigl\{
    \dist{a.v}{a.g}
    +\bigl|\{ b \;|\; b \in A, b.g \in \Pi(a.v, a.g)\}\bigr|
    \Bigr\}
  \end{align*}
  Observe that $\phi=0$ means that the problem is solved.
  Furthermore, for each iteration of Lines~\ref{algo:offline:loop-start}--\ref{algo:offline:loop-end}, $\phi$ is non-increasing.
  We now proof that $\phi$ decreases for each iteration only when $\phi > 0$ by contradiction.

  Suppose that $\phi (\neq 0)$ does not differ from the last iteration.
  Since $\phi \neq 0$, there are agents not on their targets. Let them be $B \subseteq A$.
  First, there is no swap operation by Line~\ref{algo:offline:swap}; otherwise, the second term of $\phi$ must decrease.
  Second, all agents in $B$ do not move; otherwise, the first term of $\phi$ must decrease.
  Furthermore, for an agent $a \in B$, \nextnode{a.v}{a.g}, let denote this as $a.u$, must be occupied by another agent $b \in B$; otherwise, $a$ moves to $a.u$.
  This is the same for $b$, i.e., there is an agent $c \in B$ such that $c.v = b.u$.
  By induction, this sequence of agents must form a deadlock somewhere;
  however, by deadlock detection and resolution in Lines~\ref{algo:offline:deadlock-detection}--\ref{algo:offline:rotation}, both the first and second terms of $\phi$ must decrease.
  Hence, this is a contradiction.
\end{proof}

Any assignment algorithms can be applied to TSWAP because Theorem~\ref{thrm:offline} does not rely on initial assignments.
Furthermore, TSWAP can easily adapt to unlabeled-MAPF instances with $|A| > |\mathcal{T}|$ by assigning agents without targets to any non-target locations.

\begin{proposition}
  Offline TSWAP has upper bounds of;
  \begin{itemize}
  \item makespan: $O(A\cdot\diam(G))$
  \item sum-of-costs: $O(A^2\cdot\diam(G))$
  \item maximum-moves: $O(A\cdot\diam(G))$
  \item sum-of-moves: $O(A\cdot\diam(G))$
  \end{itemize}
  \label{prop:offline:upper-bound}
\end{proposition}

Compared to sum-of-costs, the upper bound on sum-of-moves is significantly reduced because it ignores all ``wait'' actions.
The bound on sum-of-moves is tight in some scenarios, such as a line graph with all agents starting on one end and all targets on the opposite end.
In general, the upper bound on makespan is greatly overestimated.
We see empirically later that TSWAP yields near-optimal solutions for makespan depending on initial assignments.

\begin{proposition}
  Assume that the time complexity of $\mathsf{nextNode}$ and the deadlock resolution [Lines~\ref{algo:offline:deadlock-detection}--\ref{algo:offline:rotation}] in Alg.~\ref{algo:offline} are $\alpha$ and $\beta$, respectively.
  The time complexity of Offline TSWAP excluding Line~\ref{algo:offline:assign} is $O(A^2\cdot\diam(G)\cdot(\alpha\!+\!\beta))$.
  \label{prop:complexity-all}
\end{proposition}

From Proposition~\ref{prop:complexity-all}, TSWAP has an advantage in large fields compared to the time complexity $O(AV^2)$ of the makespan-optimal algorithm~\cite{yu2013multi} with a natural assumption that $E=O(V)$.

\section{Target Assignment with Lazy Evaluation}
\label{sec:target-assignment}
Target assignment determines where to go; takes an unlabeled-MAPF instance as input, then returns an assignment, a set of pairs of an initial location and a target, as output.
An initial assignment is crucial for TSWAP.
Ideal assignment algorithms are quick, scalable, and with reasonable quality for solution metrics (e.g., makespan).
It is possible to apply conventional assignment algorithms~\cite{kuhn1955hungarian}.
However, costs (i.e., distances) for each start-target pair are unknown initially, which is typically computed via breadth-first search with time complexity $O(A(V+E))$.
This would be a non-negligible overhead.
We thus present two examples that efficiently solve target assignment with \emph{lazy evaluation}, which avoids exhaustive distance evaluation.

\subsection{Bottleneck Assignment}
\begin{algorithm}[tb]
  \caption{\textbf{Bottleneck Assignment}}
  \label{algo:target-assignment}
  {\small
  \begin{algorithmic}[1]
    \item[\textbf{input}:~unlabeled-MAPF instance]
    \item[\textbf{output}:~$\mathcal{M}$: assignment, a set of pairs $s \in \mathcal{S}$ and $g \in \mathcal{T}$]
    \STATE initialize $\mathcal{M}$; Let $\mathcal{B}$ be a bipartite graph $(\mathcal{S}, \mathcal{T}, \emptyset)$
    \label{algo:ta:bipartite}
    \STATE $\mathcal{Q}$~:~priority queue of tuple\\
    ~~~~~~$s \in \mathcal{S}$, $g \in \mathcal{T}$, real distance, and estimated distance\\
    ~~~~~~in increasing order of distance\\
    ~~~~~~(use real one if exists, otherwise use estimated one)
    \label{algo:ta:priority-queue}
    \vspace{0.1cm}
    \STATE $\mathcal{Q}.\text{push}\left(\left(s, g, \bot, h(s, g)\right)\right)$ : for each pair $s\in\mathcal{S}, g\in\mathcal{T}$
    \label{algo:ta:setup}
    \WHILE{$\mathcal{Q} \neq \emptyset$}
    \label{algo:ta:start-loop}
    \STATE $(s, g, d, \Delta) \leftarrow \mathcal{Q}.\text{pop()}$
    \IF{$d = \bot$}
    \STATE $\mathcal{Q}.\text{push}\left((s, g, \dist{s}{g}, \Delta)\right)$; \textbf{continue}
    \label{algo:ta:lazy-eval}
    \ENDIF
    \STATE add a new edge $(s, g)$ to $\mathcal{B}$
    \label{algo:ta:add-new-edge}
    \STATE update $\mathcal{M}$ by finding an augmenting path on $\mathcal{B}$
    \label{algo:ta:update-matching}
    \IF{$|\mathcal{M}| = |\mathcal{T}|$}
    \label{algo:ta:check-end}
    \STATE $^\dagger$optional: add all $(s^\prime, g^\prime, d^\prime, \cdot) \in \mathcal{Q}$ to $\mathcal{B}$ s.t. $d^\prime = d$
    \label{algo:ta:additional-edge}
    \STATE \textbf{break}
    \ENDIF
    \ENDWHILE
    \label{algo:ta:end-bap}
    \STATE $^\dagger$optional: $\mathcal{M} \leftarrow$ minimum cost maximum matching on $\mathcal{B}$
    \label{algo:ta:mincost-matching}
  \end{algorithmic}
  }
\end{algorithm}
Algorithm~\ref{algo:target-assignment} aims to minimize makespan by solving the bottleneck assignment problem~\cite{gross1959bottleneck}, i.e., assign each agent to one target while minimizing the maximum cost, regarding distances between initial locations and targets as costs.

The algorithm incrementally adds pairs of initial location and target to a bipartite graph $\mathcal{B}$ [Line~\ref{algo:ta:add-new-edge}], in increasing order of their distances using a priority queue $\mathcal{Q}$.
$\mathcal{B}$ is initialized as $(\mathcal{S}, \mathcal{T}, \emptyset)$ [Line~\ref{algo:ta:bipartite}].
This iteration continues until all targets are matched to initial locations, i.e., agents [Line~\ref{algo:ta:check-end}].
At each iteration, the maximum bipartite matching problem on $\mathcal{B}$ is solved [Line~\ref{algo:ta:update-matching}].
In general, the Hopcroft-Karp algorithm~\cite{hopcroft1973n} efficiently solves this problem in $O(\sqrt{V^\prime}E^\prime)$ runtime for any bipartite graph $(V^\prime,E^\prime)$, but we use the reduction to the maximum flow problem and the Ford-Fulkerson algorithm~\cite{ford1956maximal}.
The basic concept of this algorithm is finding repeatedly an \emph{augmenting path}, i.e., a path from source to sink with available capacity on all edges in the path, then making the flow along that path.
Such paths are found, e.g., via depth-first or breadth-first search.
Here, finding a single augmenting path in  $O(E^\prime)$ runtime is sufficient to update the matching because the number of matched pairs increases at most once for each adding.

The algorithm uses \emph{lazy evaluation} of real distance [Line~\ref{algo:ta:lazy-eval}].
We use the priority queue $\mathcal{Q}$ [Line~\ref{algo:ta:priority-queue}] and admissible heuristics $h$ [Line~\ref{algo:ta:setup}], then evaluate the real distance as needed.
$\bot$ denotes that the corresponding real distance has not been evaluated yet.
The lazy evaluation contributes to speedup of the target assignment, as we will see later.

The algorithm \emph{optionally} solves the minimum cost maximum matching problem [Line~\ref{algo:ta:mincost-matching}], aiming at improving the sum-of-costs metric.
The problem can be solved by reducing to the minimum cost maximum flow problem then using the successive shortest path algorithm~\cite{ahuja1993network}.
Note that when finding the bottleneck cost, all edges in $\mathcal{Q}$ with their costs equal to the bottleneck cost are added to $\mathcal{B}$ to improve the sum-of-costs metric of the assignment [Line~\ref{algo:ta:additional-edge}].
This operation includes lazy evaluation similar to the main loop [Line~\ref{algo:ta:start-loop}--\ref{algo:ta:end-bap}].
We denote the corresponding algorithm as Alg.~\ref{algo:target-assignment}$^\dagger$.

\begin{proposition}
  The time complexity of Algorithm~\ref{algo:target-assignment}$^{(\dagger)}$ is $O\left(\max\left(A(V\!+\!E), A^4\right)\right)$.
  \label{prop:complexity-assignment}
\end{proposition}

\subsection{Greedy Assignment with Refinement}
\begin{algorithm}[tb]
  \caption{\textbf{Greedy Assignment with Refinement}}
  \label{algo:target-greedy}
  {\small
  \begin{algorithmic}[1]
    \item[\textbf{input}:~unlabeled-MAPF instance]
    \item[\textbf{output}:~$\mathcal{M}$: assignment, a set of pairs $s \in \mathcal{S}$ and $g \in \mathcal{T}$]
      \STATE initialize $\mathcal{M}$; initialize queue $\mathcal{U}$ by $A$
      \label{algo:greedy:init}
      \WHILE{$\mathcal{U} \neq \emptyset$}
      \STATE $a_i \leftarrow \mathcal{U}.\text{pop}()$
      \WHILE{true}
      \STATE $g \leftarrow$ the non-evaluated nearest target from $s_i$
      \label{algo:greedy:obtain-g}
      \IF{$\not\exists (s_j, g) \in \mathcal{M}$}
      \STATE add $(s_i, g)$ to $\mathcal{M}$; \textbf{break}
      \ELSIF{$\exists (s_j, g) \in \mathcal{M} \land \dist{s_i}{g} < \dist{s_j}{g}$}
      \label{algo:greedy:eval-pair}
      \STATE replace $(s_j, g) \in \mathcal{M}$ by $(s_i, g)$; $\mathcal{U}.\text{push}(a_j)$;
      \textbf{break}
      \ENDIF
      \label{algo:greedy:fin-operation}
      \ENDWHILE
      \ENDWHILE
      \label{algo:greedy:fin-init-assign}
      \smallskip
      \WHILE{$\mathcal{M}$ is updated in the last iteration}
      \label{algo:greedy:start-refine}
      \STATE $(s_i, g_i) \leftarrow \argmax_{(s, g) \in \mathcal{M}}\dist{s}{g}$;
      $c_{\text{now}} \leftarrow \dist{s_i}{g_i}$
      \label{algo:greedy:argmax}
      \FOR{$(s_j, g_j) \in \mathcal{M}$}
      \label{algo:greedy:start-try-swap}
      \IFSINGLE{$h(s_j, g_i) \geq c_{\text{now}}$}{\textbf{continue}}
      \COMMENT{for lazy evaluation}
      \STATE $c_{\text{swap}} \leftarrow \max\left(\dist{s_j}{g_i}, \dist{s_i}{g_j}\right)$
      \IFSINGLE{$c_{\text{swap}} < c_{\text{now}}$}{swap $g_i$ and $g_j$ of $\mathcal{M}$; \textbf{break}}
      \ENDFOR
      \label{algo:greedy:fin-try-swap}
      \ENDWHILE
      \label{algo:greedy:end-refine}
  \end{algorithmic}
  }
\end{algorithm}
Algorithm~\ref{algo:target-greedy} aims at finding a reasonable assignment for makespan as quickly as possible, which uses;
\begin{itemize}
\item \emph{greedy assignment} [Lines~\ref{algo:greedy:init}--\ref{algo:greedy:fin-init-assign}]; assigns one target to one agent step by step, while allowing reassignment if a better assignment will be expected [Lines~\ref{algo:greedy:eval-pair}--\ref{algo:greedy:fin-operation}].
\item \emph{iterative refinement} [Lines~\ref{algo:greedy:start-refine}--\ref{algo:greedy:end-refine}]; swaps targets of two agents until no improvements are detected.
\item \emph{lazy evaluation} of distances for start-target pairs, implemented by pausing the breadth-first search as soon as the query start-target pair is in the search tree.
\end{itemize}

This algorithm is expected to run in a very short time;
\begin{proposition}
  The time complexity of Algorithm~\ref{algo:target-greedy} is $O(A(V+E))$.
  \label{prop:complexity-greedy}
\end{proposition}

Algorithm~\ref{algo:target-greedy} presents the refinement for makespan but the refinement for sum-of-costs is straightforward (see Alg.~\ref{algo:target-greedy-soc} in the Appendix).

\section{Evaluation of Offline Planning}
\label{sec:evaluation}
The experiments aim at demonstrating that Offline TSWAP is efficient, i.e., it returns near-optimal solutions within a short time and scales well, depending on initial assignments.
In particular, this section has three aspects:
(1)~illustrating the effect of initial assignments including the proposed assignment algorithms,
(2)~comparing with the makespan-optimal polynomial-time algorithm~\cite{yu2013multi},
(3)~assessing another metric, sum-of-costs.
We carefully picked up several 4-connected grids from MAPF benchmarks~\cite{stern2019def} as a graph $G$, shown in Fig.~\ref{fig:maps}; they are common in MAPF studies.
The simulator was developed in C++ and the experiments were run on a laptop with Intel Core~i9 \SI{2.3}{\giga\hertz} CPU and \SI{16}{\giga\byte} RAM.
For each setting, we created $50$~instances while randomly generating starts and targets.
Implementation details of \cite{yu2013multi} are described in the Appendix.
Throughout this section, the runtime evaluation of TSWAP includes both target assignment and path planning.

{
  \setlength{\tabcolsep}{1pt}
  \newcommand{\colwidth}{0.18\hsize}
  \newcommand{\imgwidth}{0.9\hsize}
  \newcommand{\col}[3]{
     \begin{minipage}{\colwidth}
       \centering
       {\tiny\textit{#1}}\\
       \includegraphics[width=\imgwidth]{fig/raw/#1.pdf}\\
       {\vspace{-0.2cm}\tiny\m{#2}}\\
       {\vspace{-0.15cm}\tiny\m{(#3)}}
     \end{minipage}
  }
  \begin{figure}[ht!]
    \advance\leftskip-0.12cm
    \begin{tabular}{ccccc}
      \col{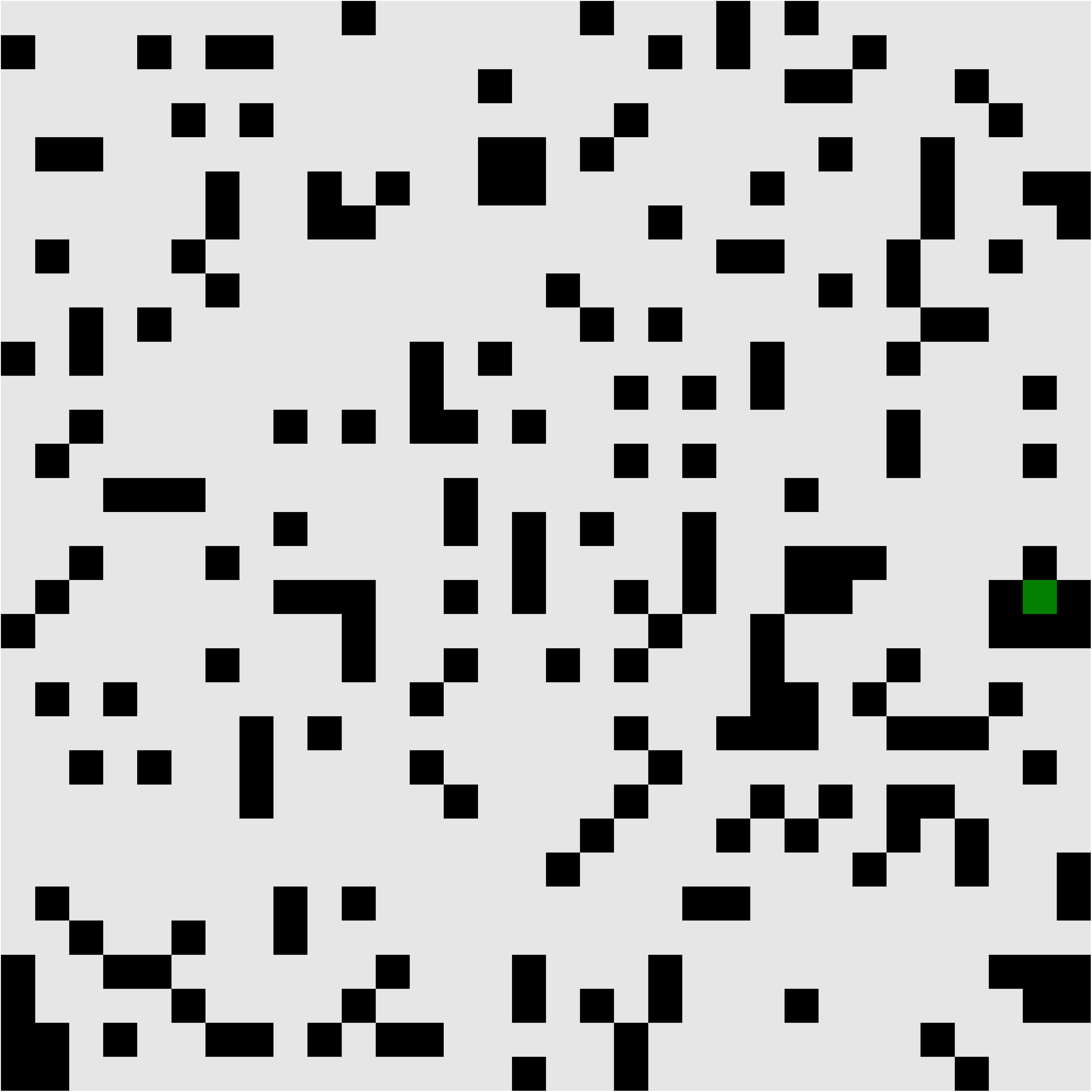}{32\times 32}{819}
      \col{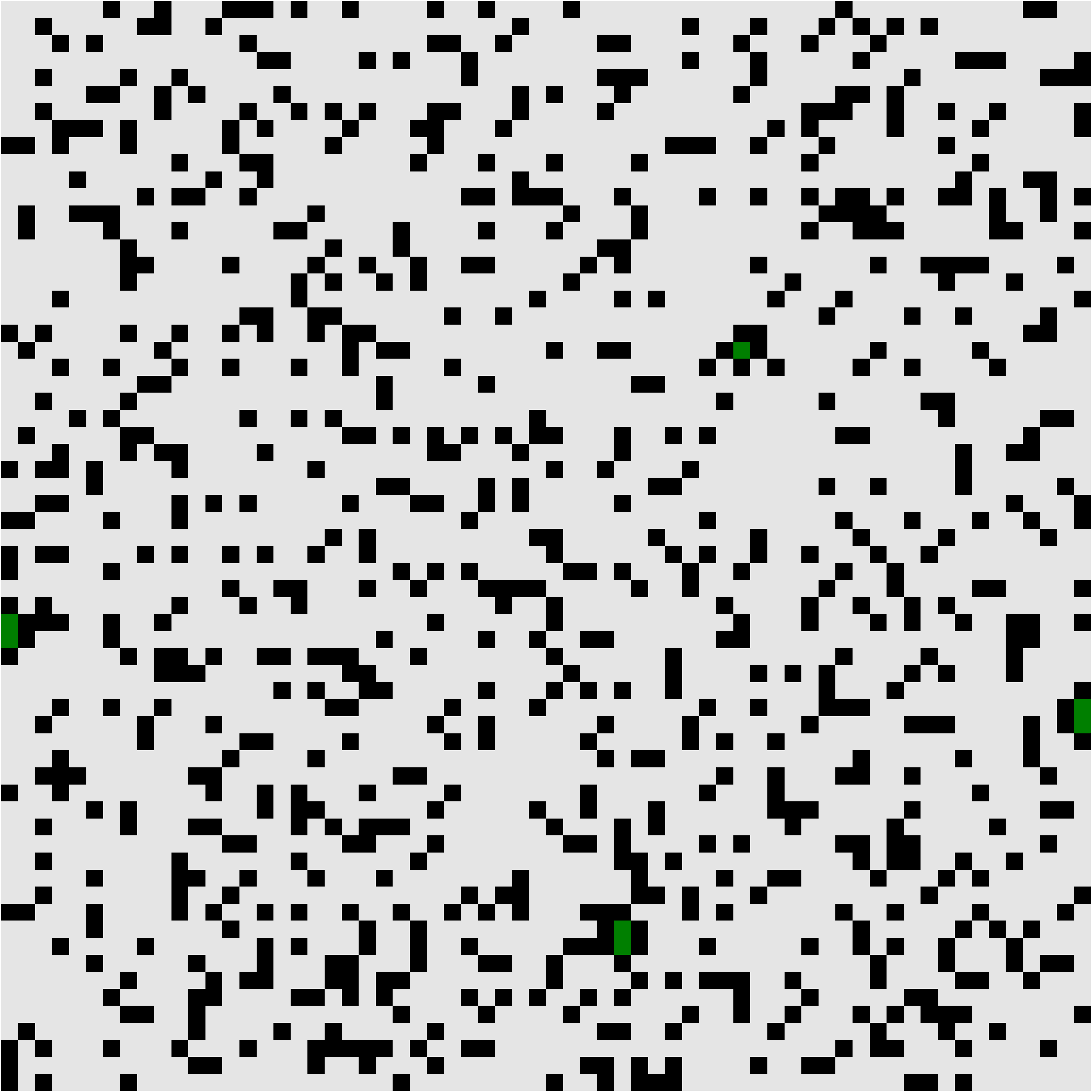}{64\times 64}{3,270}
      \col{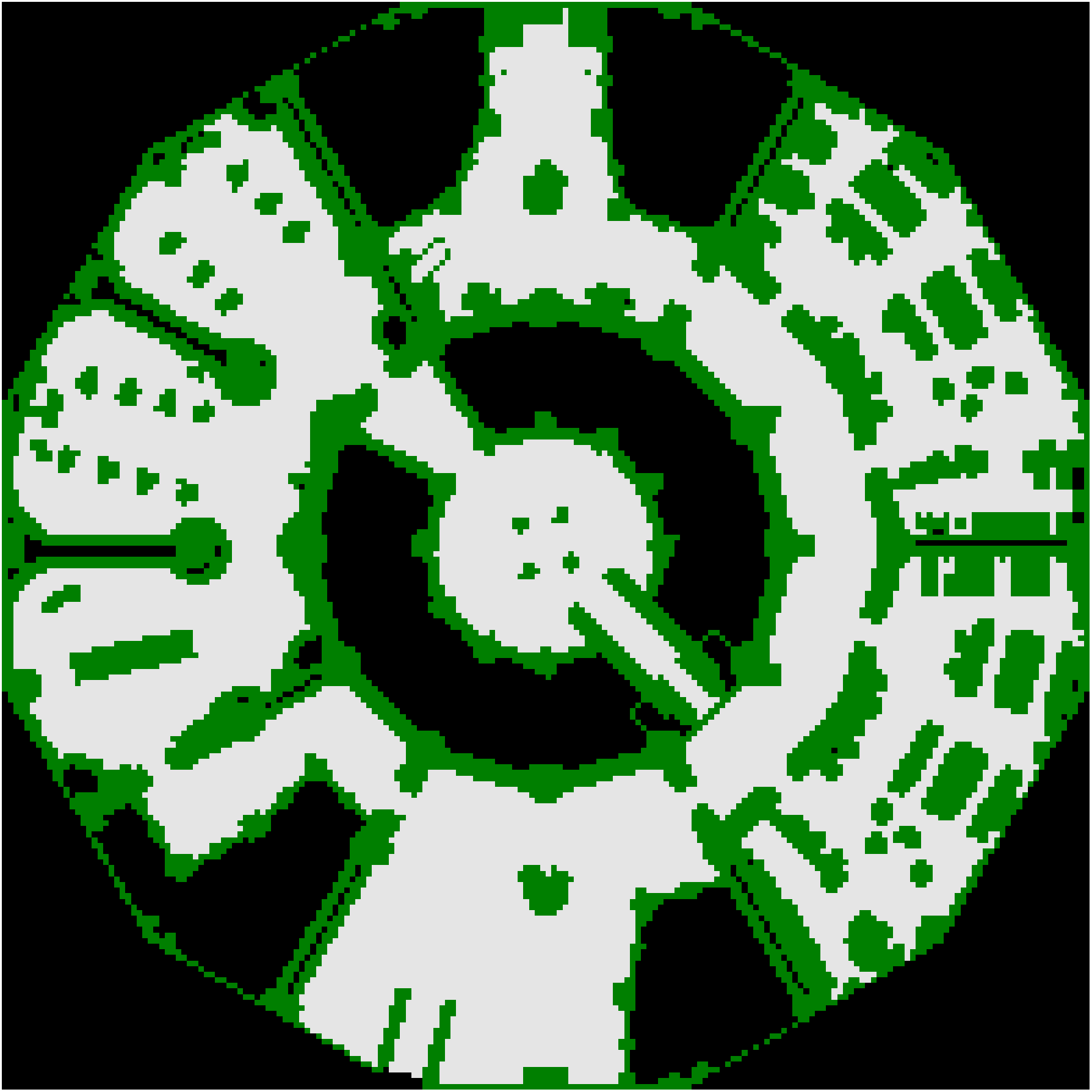}{194\times 194}{14,784}
      \col{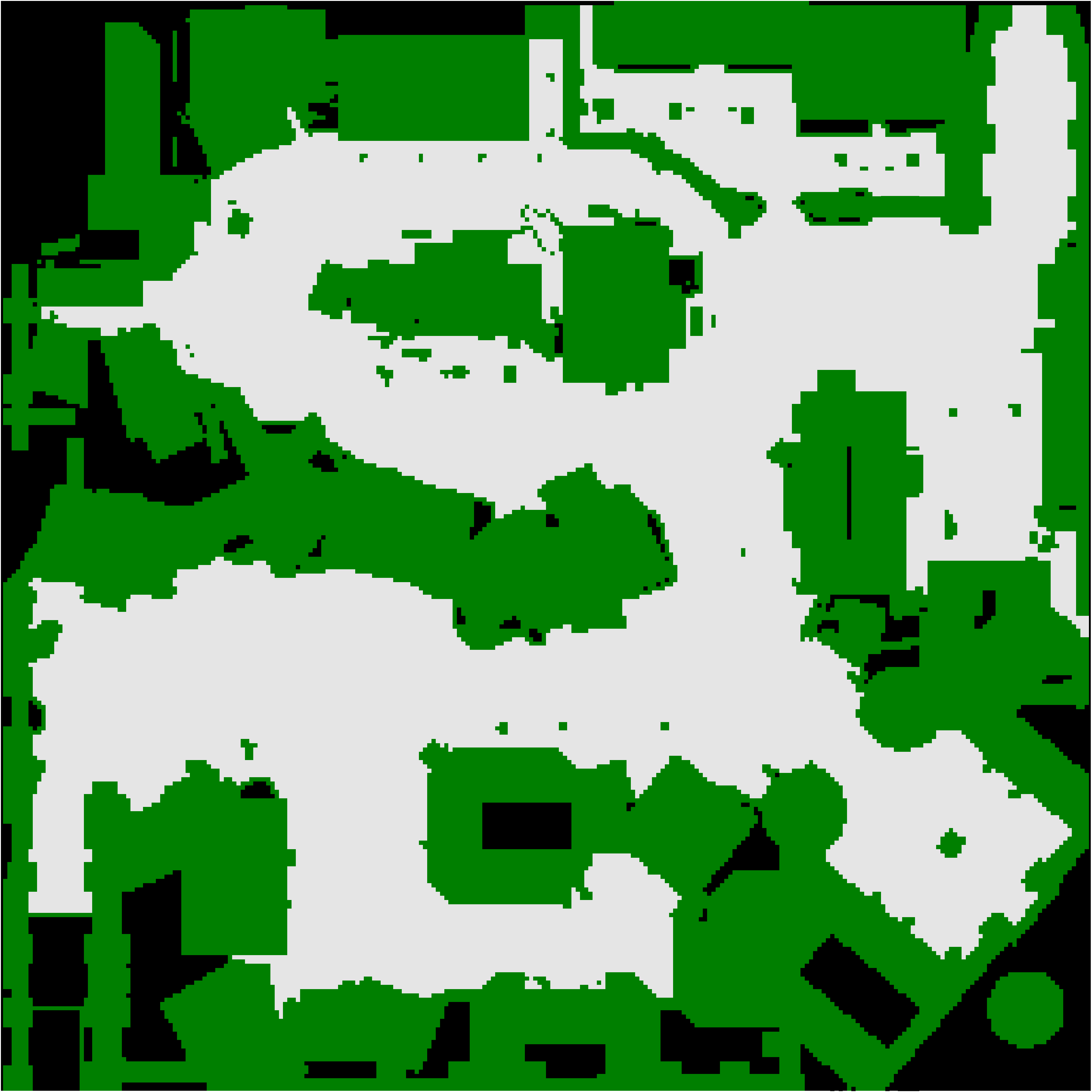}{257\times 256}{28,178}
      \col{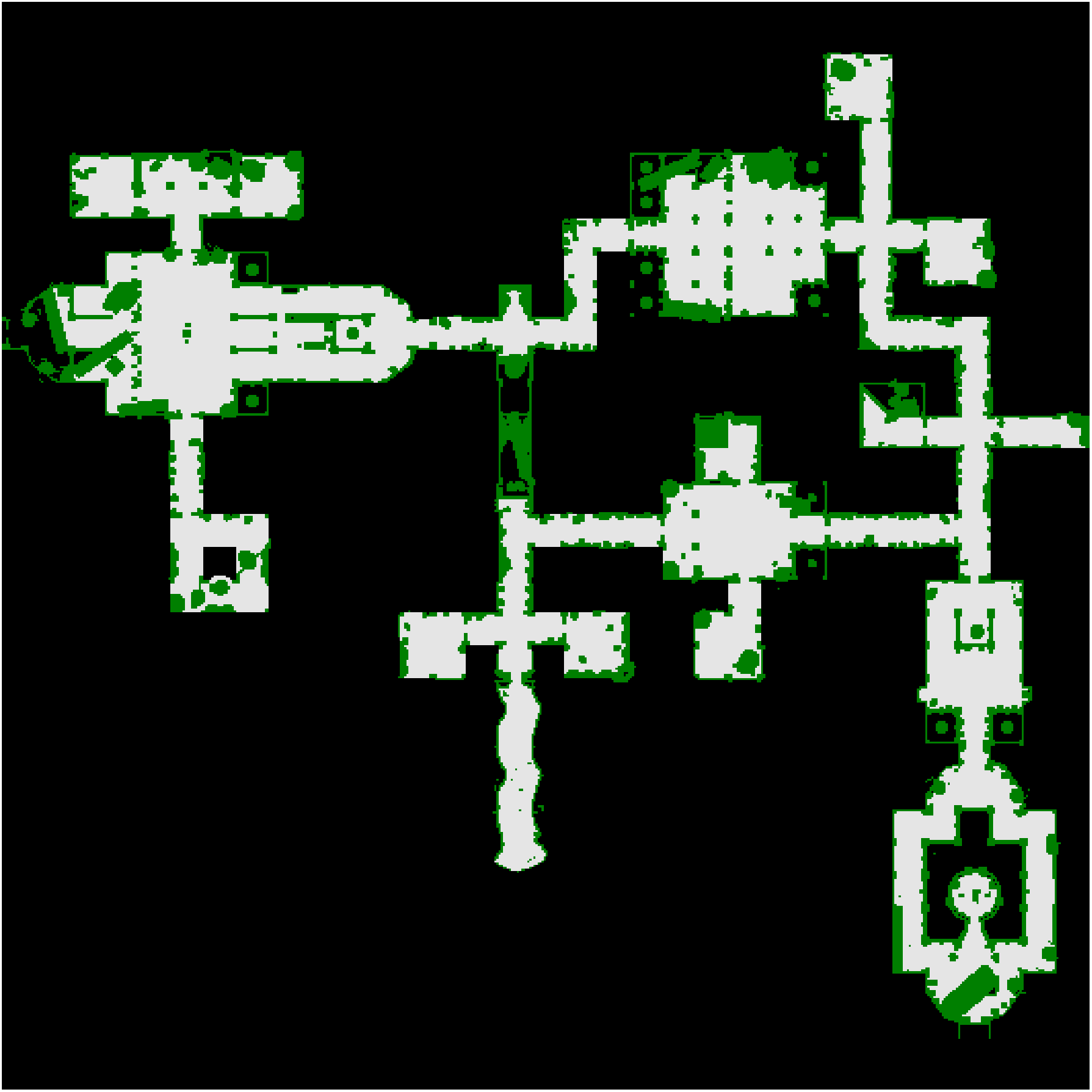}{418\times 530}{43,151}
    \end{tabular}
    \caption{\textbf{Used maps.} $|V|$ is shown with parentheses.}
    \label{fig:maps}
  \end{figure}
}

\subsection{Effect of Initial Target Assignment}
\label{sec:eval:assign}
{
  \renewcommand{\arraystretch}{0.9}
\begin{table}[t!]
  \centering
  {
    \newcommand{\cmid}{\cmidrule(lr){1-10}}
    \newcommand{\w}[1]{\textbf{#1}}  
    \setlength{\tabcolsep}{1.7pt}
    \scriptsize
    \begin{tabular}{rrrrrrrrrr}
      \toprule
      $|A|$ & metric
      & Alg.~\ref{algo:target-assignment}
      & Alg.~\ref{algo:target-assignment}$^\dagger$
      & Alg.~\ref{algo:target-assignment}$^{\dagger\ast}$
      & Alg.~\ref{algo:target-greedy}
      & Alg.~\ref{algo:target-greedy}$^\ast$
      & Alg.~\ref{algo:target-greedy-soc}
      & greedy
      & linear
      \\ \midrule
      \multirow{3}{*}{110}
            & runtime (ms) & 6 & 10 & 17 & \w{4} & 12 & \w{4} & 12 & 23 \\
            & makespan & \w{17} & \w{17} & \w{17} & 20 & 20 & 36 & 84 & 36\\
            & sum-of-costs & 1079 & \w{937} & \w{937} & 1139 & 1136 & \w{958} & 1396 & \w{940} \\ \cmid
      \multirow{3}{*}{500}
            & runtime (ms) & 72 & 174 & 213 & \w{15} & 78 & 22 & 77 & 602\\
            & makespan & \w{10} & \w{11} & \w{11} & 13 & 13 & 34 & 79 & 32 \\
            & sum-of-costs & 2595 & \w{2169} & \w{2169} & 2878 & 2860 & 2546 & 4653 & 2429 \\ \cmid
      \multirow{3}{*}{1000}
            & runtime (ms) & 335 & 757 & 882 & \w{28} & 233 & 58 & 221 & 3811\\
            & makespan & \w{9} & \w{9} & \w{9} & 11 & 11 & 28 & 71 & 26\\
            & sum-of-costs & 3591 & \w{2922} & \w{2922} & 4020 & 4043 & 3695 & 7571 & 3491 \\ \cmid
      \multirow{3}{*}{2000}
            & runtime (ms) & 1453 & 3035 & 3205 & \w{66} & 784 & 188 & 725 & 32944\\
            & makespan & \w{8} & \w{7} & \w{7} & 10 & 10 & 24 & 56 & 23 \\
            & sum-of-costs & 4670 & \w{3469} & \w{3469} & 5200 & 5244 & 5465 & 12292 & 5122 \\ \bottomrule
    \end{tabular}
  }
  \vspace{-0.2cm}
  \caption{\textbf{The results of TSWAP with different assignments in \textit{random-64-64-20}.}
    Bold characters are based on 95\% confidence intervals of the mean in the Appendix, which also presents the result on another map \textit{lak303d}.
  }
  \label{table:result-assignment}
\end{table}
}
The first part evaluates the effect of initial assignments on TSWAP while varying $|A|$.
We tested Alg.~\ref{algo:target-assignment} (bottleneck; minimizing maximum distance), Alg.~\ref{algo:target-assignment}$^\dagger$ (with min-cost maximum matching), Alg.~\ref{algo:target-greedy} (greedy with refinement for makespan), Alg.~\ref{algo:target-greedy-soc} (for sum-of-costs), naive greedy assignment~\cite{avis1983survey}, and optimal linear assignment (minimizing total distances) solved by the successive shortest path algorithm~\cite{ahuja1993network}.
Note that these assignment algorithms do not consider inter-agent collisions.
The last two used distances for start-target pairs obtained by the breadth-first search as costs.
To assess the effect of lazy evaluation, we also tested the adapted version of Alg.~\ref{algo:target-assignment}$^\dagger$ and Alg.~\ref{algo:target-greedy} without lazy evaluation, denoted as Alg.~\ref{algo:target-assignment}$^{\dagger\ast}$ and Alg.~\ref{algo:target-greedy}$^\ast$.

Table~\ref{table:result-assignment} summarizes the results on \textit{random-64-64-20}.
In summary, Alg.~\ref{algo:target-assignment} contributes to finding good solutions for makespan.
Algorithm~\ref{algo:target-assignment}$^\dagger$ significantly improves sum-of-costs.
Algorithm~\ref{algo:target-greedy} and Alg.~\ref{algo:target-greedy-soc} are blazing fast while solution qualities outperform those of the naive greedy assignment.
The lazy evaluation speedups each assignment algorithm.
The optimal linear assignment requires time because its time complexity is $O(A^3)$.

\subsection{Makespan-optimal Algorithm v.s. TSWAP}

{
  \newcommand{\colwidth}{0.49\hsize}
  \setlength{\tabcolsep}{0pt}
  \begin{figure}[t!]
    \centering
    \begin{tabular}{cc}
      \begin{minipage}{\colwidth}
        \centering
        \includegraphics[width=1\hsize]{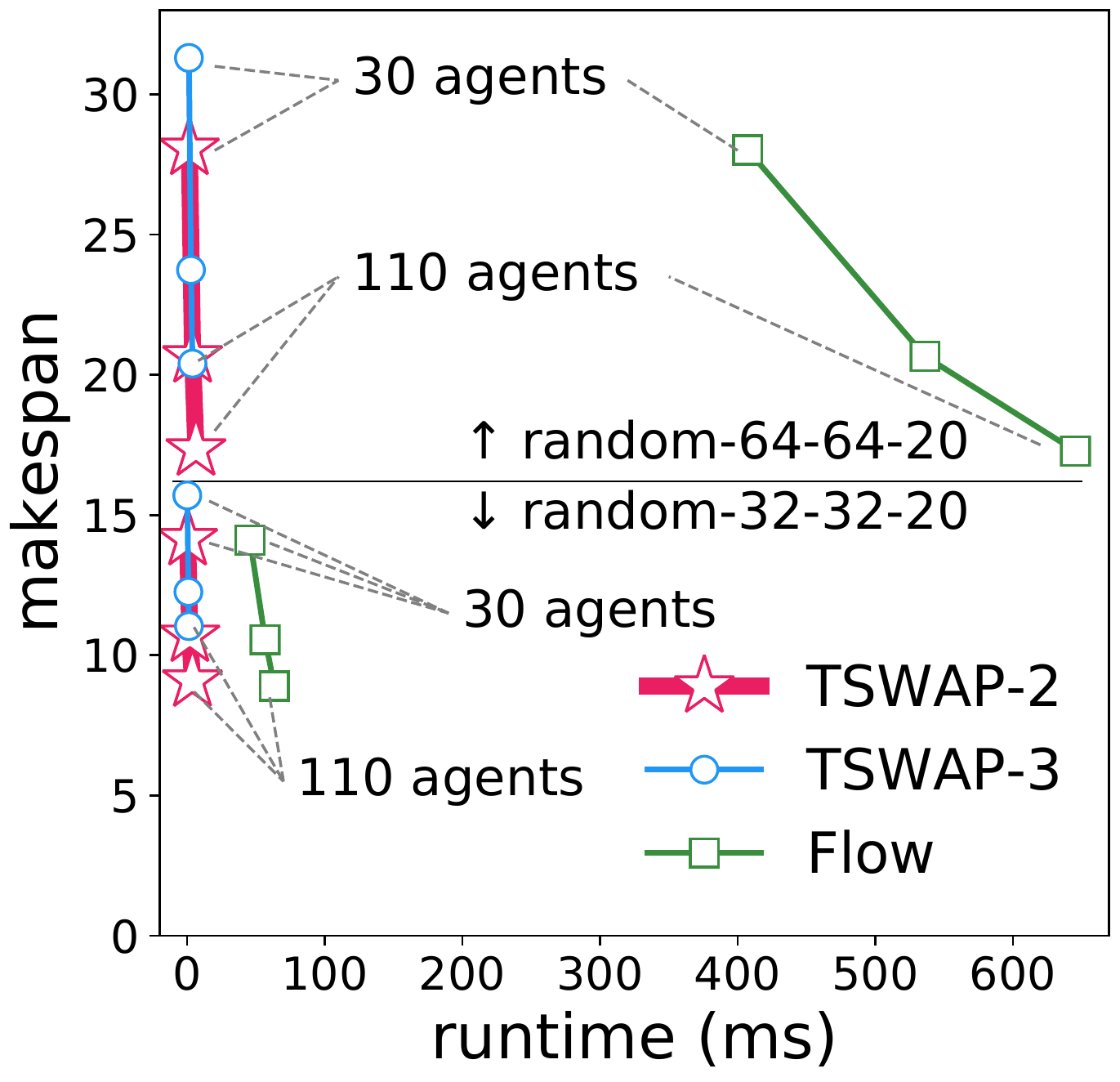}
      \end{minipage}
      \begin{minipage}{\colwidth}
        \centering
        \includegraphics[width=1\hsize]{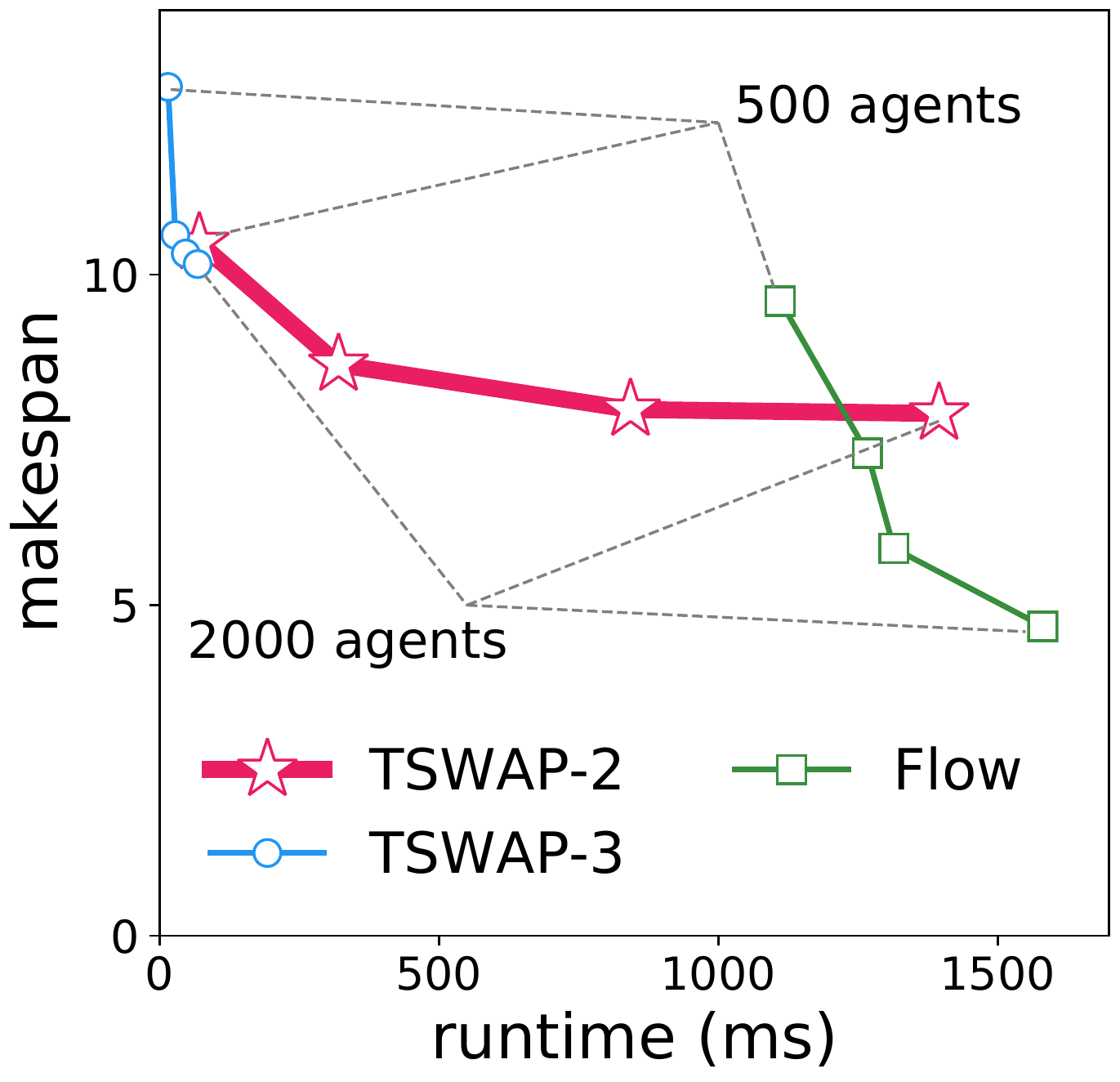}
      \end{minipage}
    \end{tabular}
    \vspace{-0.3cm}
    \caption{
      \textbf{The average makespan and runtime.}
      Alg.~\ref{algo:target-assignment} (TSWAP-2) and Alg.~\ref{algo:target-greedy} (TSWAP-3) were used in TSWAP.
      ``Flow'' is the optimal algorithm.
      \emph{left}: $|A| \in \{30, 70, 110\}$.
      \emph{right}: $|A| \in \{500, 1000, 1500, 2000\}$ on \textit{random-64-64-20}.
    }
  \label{fig:result-random}
  \end{figure}
}

This part is further divided into two:
(1)~assessing scalability for both $G$ and $A$, and
(2)~testing the solvers in large graphs.
Figure~\ref{fig:result-random} and Table~\ref{table:result-huge} summarize the results.

Figure~\ref{fig:result-random} (left) displays the average makespan and runtime of ``quadrupling'' the size of $G$, i.e., those of \textit{random-64-64-20}, regarding the results of \textit{random-32-32-20} as a baseline.
Figure~\ref{fig:result-random} (right) shows dense situations ($|A| \geq |V|/8$).
The main observations are:
(1)~TSWAP quickly yields near-optimal solutions in non-dense situations with either Alg.~\ref{algo:target-assignment} or Alg.~\ref{algo:target-greedy}.
(2)~The runtime of TSWAP remains small when enlarging $G$ while the optimal algorithm increases dramatically.
(3)~As agents increase, the runtime of TSWAP with Alg.~\ref{algo:target-assignment} quickly increases compared to the optimal algorithm, because Alg.~\ref{algo:target-assignment} is quartic on $|A|$ (see Prop.~\ref{prop:complexity-all}).
Meanwhile, TSWAP with Alg.~\ref{algo:target-greedy} immediately yields solutions even with a few thousand agents.
Note that, as the number of agents increases, the optimal makespan decreases because we set initial locations and targets randomly.

Table~\ref{table:result-huge} shows the results on large graphs with a timeout of \SI{5}{\minute}.
The optimal algorithm took time to return solutions or sometimes failed before the timeout, whereas TSWAP succeeded in all cases in a comparatively very short time, thus highlighting the need for sub-optimal algorithms of unlabeled-MAPF.
In addition, TSWAP yields high-quality solutions for the makespan.

{
  \renewcommand{\arraystretch}{0.9}
\begin{table}[t!]
  \centering
  {
    \newcommand{\cmid}{\cmidrule(lr){1-10}}
    \newcommand{\w}[1]{\textbf{#1}}  
    \setlength{\tabcolsep}{3pt}
    \scriptsize
    \begin{tabular}{rrrrrrrrrr}
      \toprule
      &&\multicolumn{3}{c}{runtime (sec)}
      & \multicolumn{3}{c}{success rate(\%)}
      & \multicolumn{2}{c}{sub-optimality}
      \\ \cmidrule(lr){3-5}\cmidrule(lr){6-8}\cmidrule(lr){9-10}
      map & $|A|$
      & Flow & Alg.~\ref{algo:target-assignment} & Alg.~\ref{algo:target-greedy}
             & Flow & Alg.~\ref{algo:target-assignment} & Alg.~\ref{algo:target-greedy}
             & Alg.~\ref{algo:target-assignment} & Alg.~\ref{algo:target-greedy}
      \\ \midrule
      \multirow{4}{*}{\textit{lak303d}}
      & 100 & 26.2 & 0.0 & \w{0.0} & 100 & 100 & 100 & 1.001 & 1.001\\
      & 500 & 60.6 & 0.6 & \w{0.1} & 100 & 100 & 100 & 1.009 & 1.022\\
      & 1000 & 54.7 & 3.5 & \w{0.2} & 100 & 100 & 100 & 1.064 & 1.073\\
      & 2000 & 56.3 & 25.4 & \w{0.4} & 100 & 100 & 100 & 1.340 & 1.358 \\ \cmid
      \multirow{4}{*}{\textit{den520d}}
      & 100 & 46.0 & 0.0 & \w{0.0} & 100 & 100 & 100 & 1.000 & 1.052\\
      & 500 & 67.5 & 0.3 & \w{0.1} & 100 & 100 & 100 & 1.003 & 1.118\\
      & 1000 & 82.3 & 1.6 & \w{0.2} & 98 & 100 & 100 & 1.014 & 1.097\\
      & 2000 & 89.8 & 9.3 & \w{0.4} & 100 & 100 & 100 & 1.043 & 1.169 \\ \cmid
      \multirow{4}{*}{\textit{brc202d}}
      & 100  & 141.9 & 0.1 & \w{0.1} & 60 & 100 & 100 & 1.000 & 1.001 \\
      & 500  & 214.5 & 0.7 & \w{0.3} & 48 & 100 & 100 & 1.001 & 1.003 \\
      & 1000 & 238.5 & 2.7 & \w{0.6} & 42 & 100 & 100 & 1.002 & 1.007 \\
      & 2000 & 230.2 & 14.8 & \w{1.1} & 16 & 100 & 100 & 1.021 & 1.026 \\
      \bottomrule
    \end{tabular}
  }
  \vspace{-0.2cm}
  \caption{\textbf{The results in large graphs.}
    ``Flow'' is the optimal algorithm.
    The scores are averages over instances that were solved by all solvers.
    The sub-optimality is for makespan, dividing the makespan of TSWAP by the optimal scores.
  }
  \label{table:result-huge}
\end{table}
}
\begin{figure}[t]
  \centering
  \includegraphics[width=0.95\hsize]{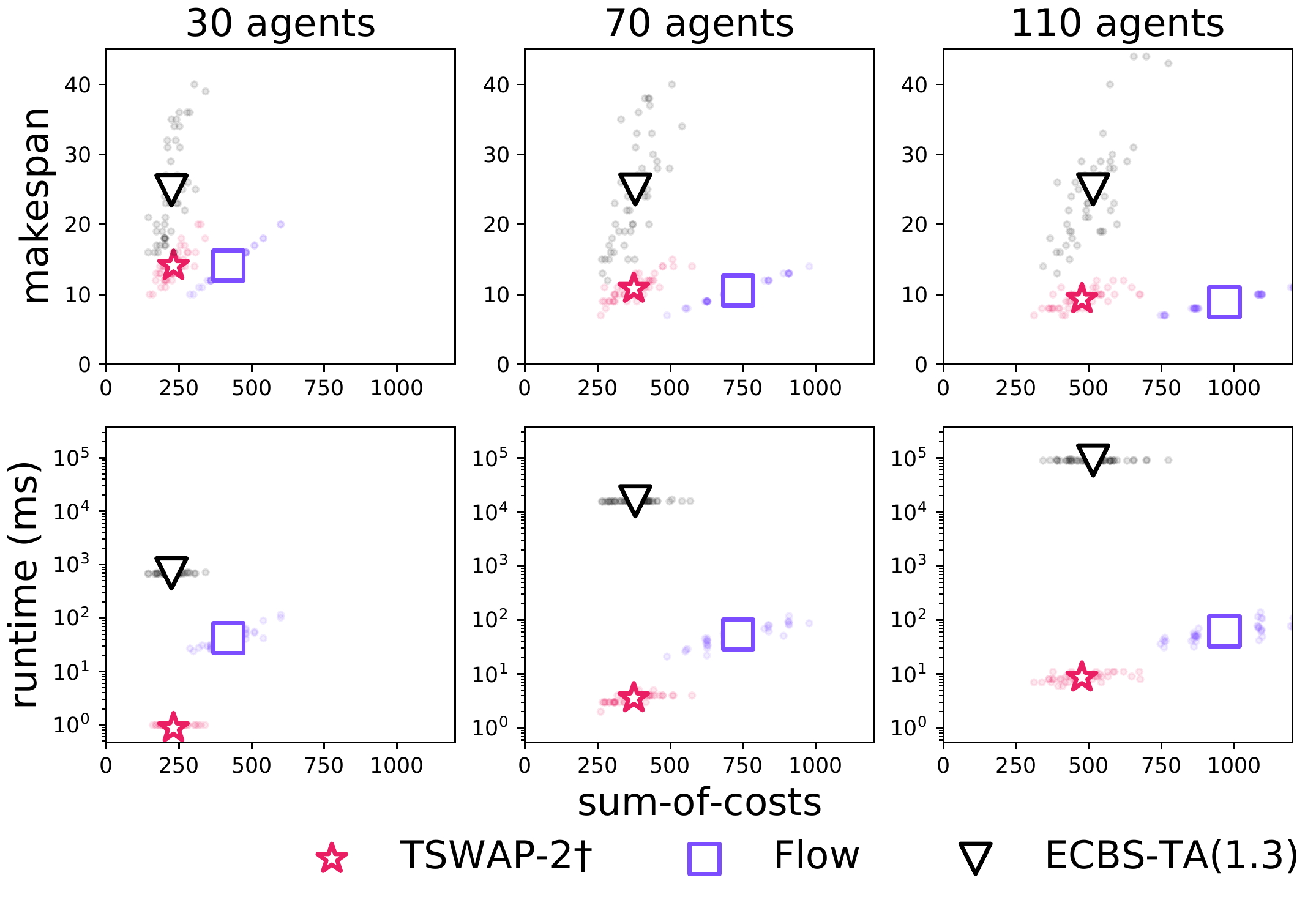}
  \vspace{-0.3cm}
  \caption{\textbf{The results for the sum-of-costs metric.}
    We used \textit{random-32-32-20}.
    The average scores are plotted.
    We also show scatter plots of 50 instances by transparent points.
  }
  \label{fig:result-soc}
\end{figure}

\subsection{Sum-of-costs Metric}
We next evaluate the sum-of-costs metric.
As a baseline, we used ECBS-TA~\citep{honig2018conflict}, which yields bounded sub-optimal solutions with respect to the sum-of-costs.
The implementation of ECBS-TA was obtained from the authors.
We used \textit{random-32-32-20} with 30, 70, and 110 agents.
The sub-optimality of ECBS-TA was set to \num{1.3}, which was adjusted to solve problems within the acceptable time (\SI{5}{\minute}).
TSWAP used Alg.~\ref{algo:target-assignment}$^\dagger$.
Note that we preliminary confirmed that ECBS-TA in denser situations failed to return solutions within a reasonable time.

Figure~\ref{fig:result-soc} shows that TSWAP yields solutions with acceptable quality while reducing computation time (lower, vertical axis) by orders of magnitude compared to the others;
the quality of sum-of-costs (horizontal axis) is competitive with ECBS-TA, with makespan quality close to optimal (upper, vertical axis).
TSWAP is significantly faster than ECBS-TA because, unlike ECBS-TA, TSWAP uses a one-shot target assignment and a simple path planning process.

\section{Online TSWAP}
\label{sec:online-tswap}
TSWAP is not limited to offline planning and can also adapt to online planning with timing uncertainties.
Algorithm~\ref{algo:online} presents \emph{Online TSWAP} to solve the online time-independent problem.
Before execution, TSWAP assigns targets to each agent [Lines~\ref{algo:online:matching}--\ref{algo:online:assign}].
During execution, the online version runs the procedure of the offline version for one agent [Line~\ref{algo:online:main}].
If the virtual location (i.e., $a.v$) is updated, then let the agent \emph{actually} moves there [Line~\ref{algo:online:move}].

{
  \renewcommand{\hl}[1]{\n{\textcolor{blue}{#1}}}
  \begin{algorithm}[tb!]
    \caption{\textbf{Online TSWAP}}
    \label{algo:online}
    {\small
    \begin{algorithmic}[1]
    \item[\textbf{input}:~unlabeled-MAPF instance]
    \item[\textbf{offline phase}]
      \STATE get an initial assignment $\mathcal{M}$
      \label{algo:online:matching}
      \STATE $a_i.v, a_i.g \leftarrow (s_i, g) \in \mathcal{M}$~:~for each agent $a_i \in A$
      \label{algo:online:assign}
      \vspace{0.1cm}
    \item[\textbf{online phase; when $a \in A$ is activated}]
      \STATE execute Lines~\ref{algo:offline:at-goal}--\ref{algo:offline:end-main} in Algorithm~\ref{algo:offline}
      \label{algo:online:main}
      \IFSINGLE{$a.v$ is updated}{move $a$ to $a.v$}
      \label{algo:online:move}
    \end{algorithmic}
    }
  \end{algorithm}
}
{
  \begin{figure*}
    \centering
    \includegraphics[width=1.0\hsize]{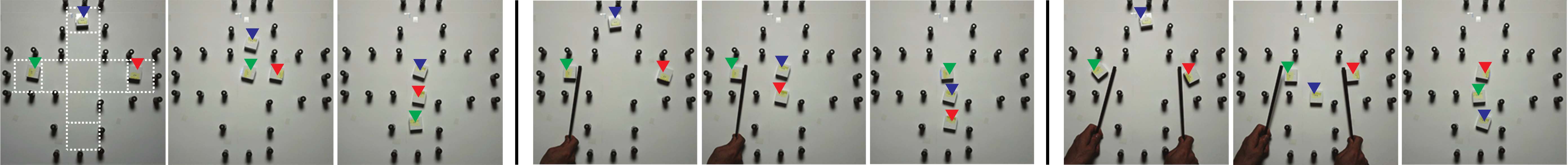}
    \vspace{-0.5cm}
    \caption{
      \textbf{Demo of time-independence of Online TSWAP.}
      We prepared three scenarios with identical initial (left for each) and terminal (right for each) configurations.
      A graph is illustrated on the leftmost image.
      We further illustrated colored triangles to distinguish each robot.
      Although the experimenter disturbed robots' progression with chopsticks during the execution (middle and right settings), all robots reach the targets while flexibly swapping their assigned targets.
    }
    \label{fig:time-independence}
  \end{figure*}
}

{
  \setlength{\tabcolsep}{2pt}
  \begin{figure*}
    \centering
    \begin{tabular}{cc|cc|cc}
      \begin{minipage}{0.16\hsize}
        \centering
        \includegraphics[width=0.9\hsize]{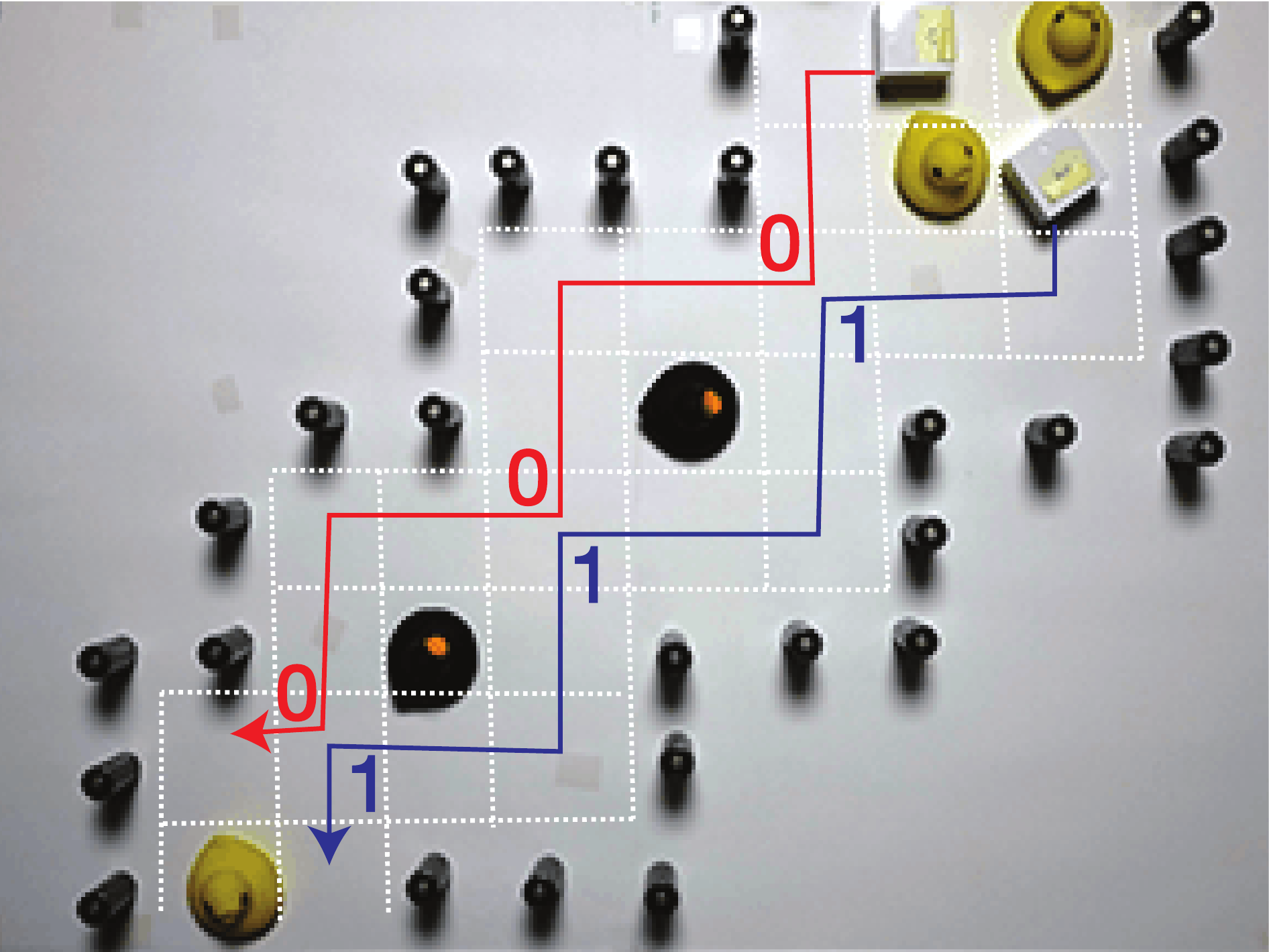}
      \end{minipage}
      &
      \begin{minipage}{0.16\hsize}
        \centering
        \includegraphics[width=1\hsize]{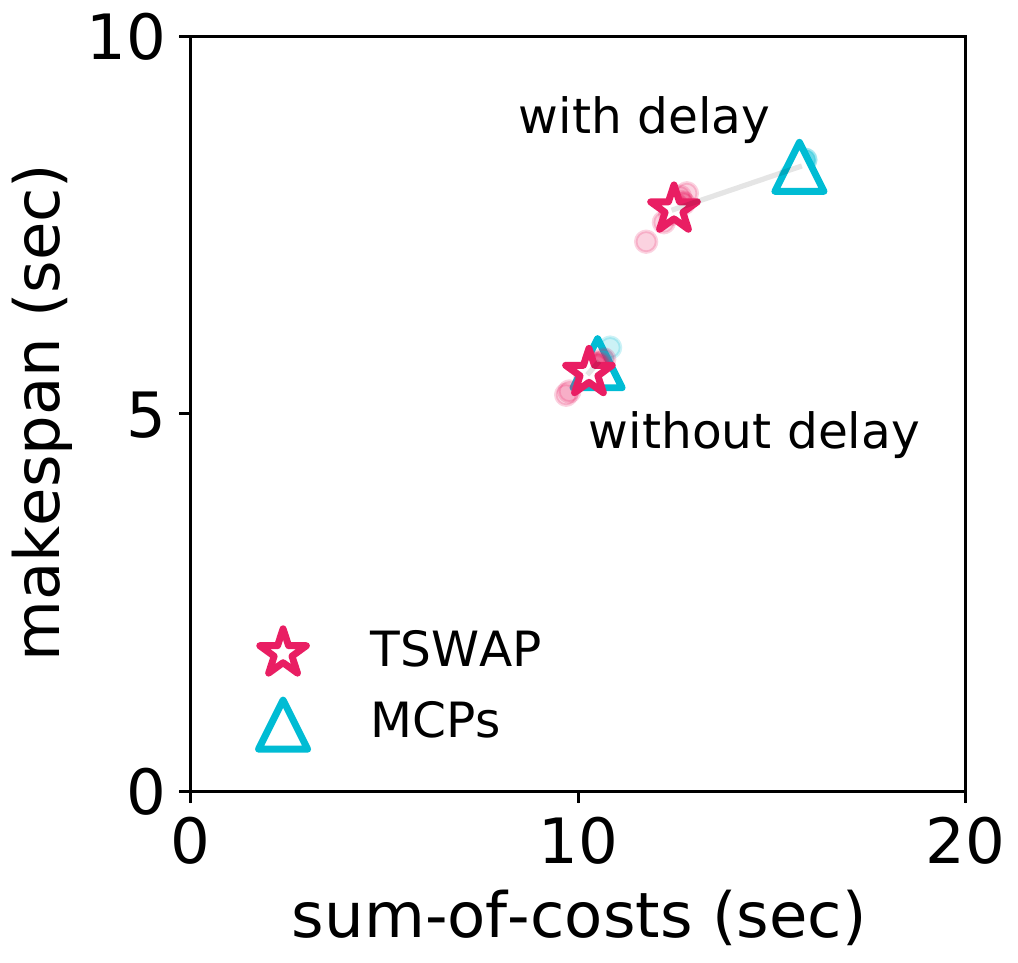}
      \end{minipage}
      &
      \begin{minipage}{0.16\hsize}
        \centering
        \includegraphics[width=0.9\hsize]{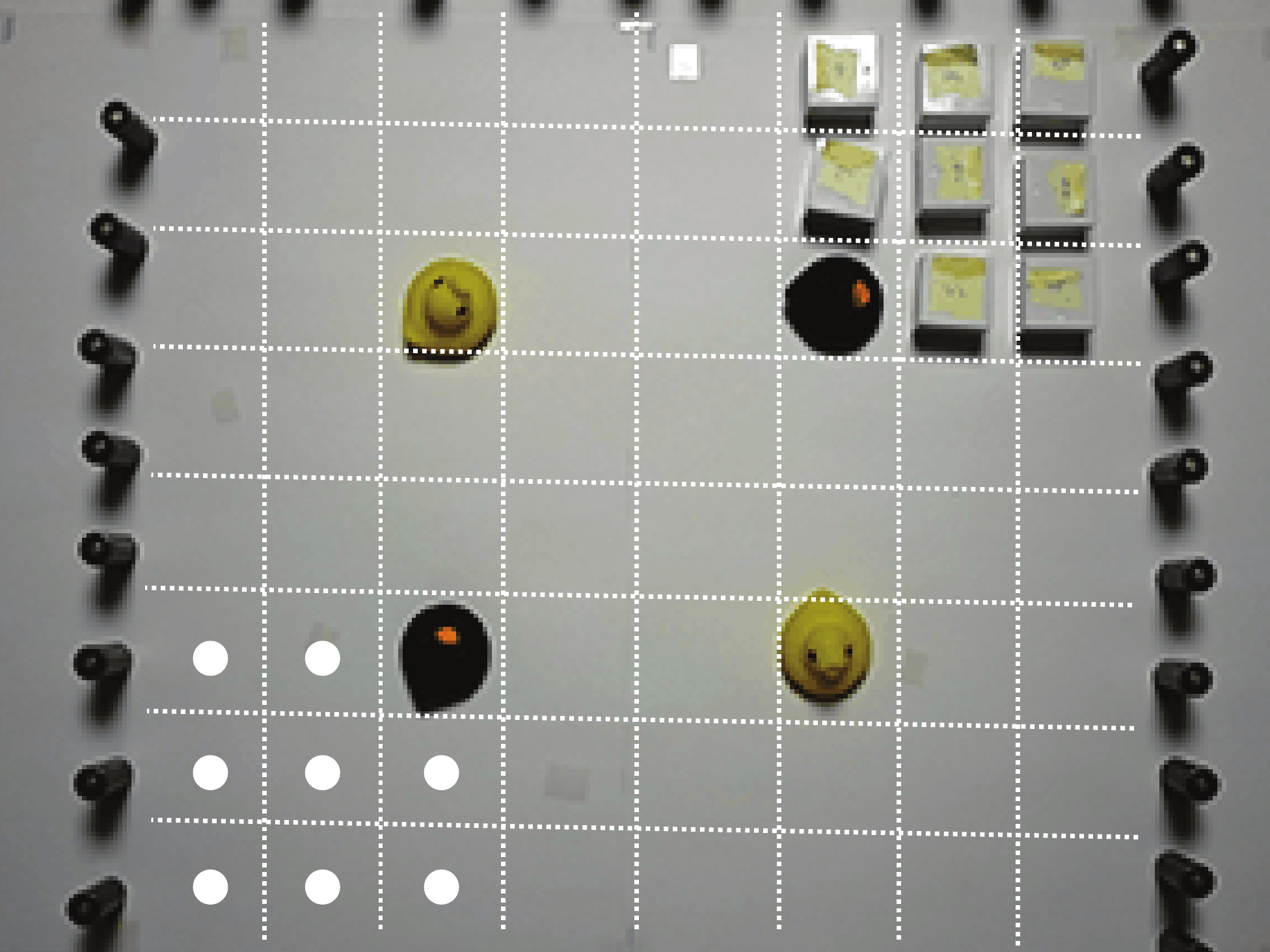}
      \end{minipage}
      &
      \begin{minipage}{0.16\hsize}
        \centering
        \includegraphics[width=1\hsize]{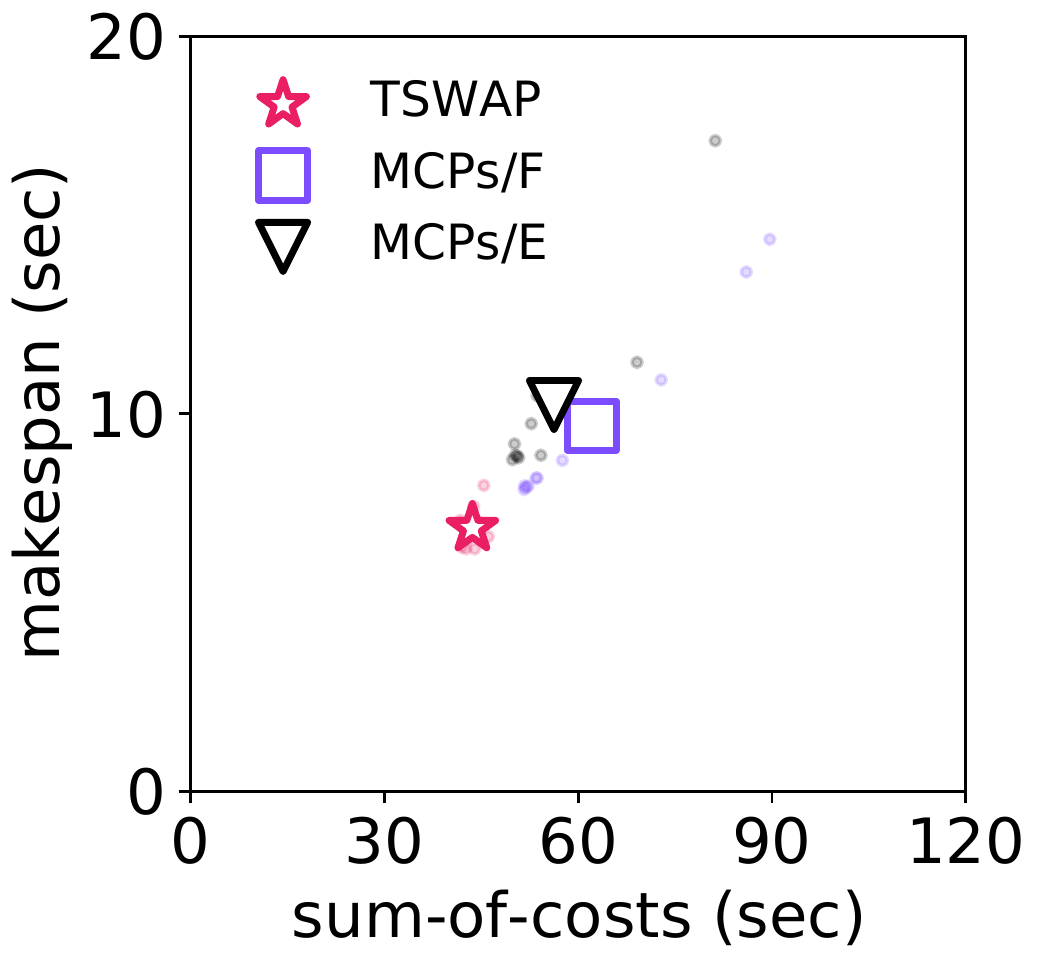}
      \end{minipage}
      &
      \begin{minipage}{0.16\hsize}
        \centering
        \includegraphics[width=0.9\hsize]{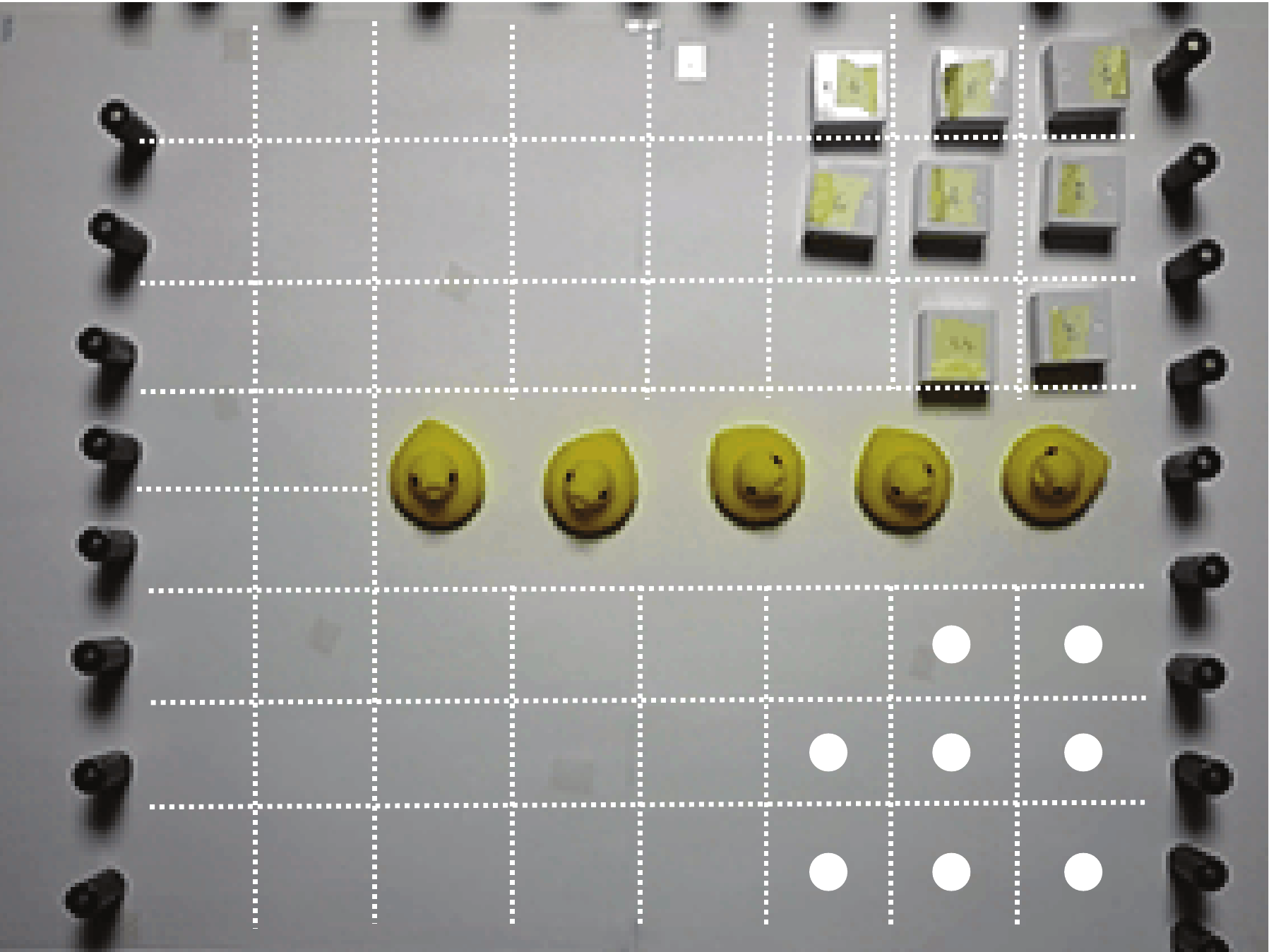}
      \end{minipage}
      \begin{minipage}{0.16\hsize}
        \centering
        \includegraphics[width=1\hsize]{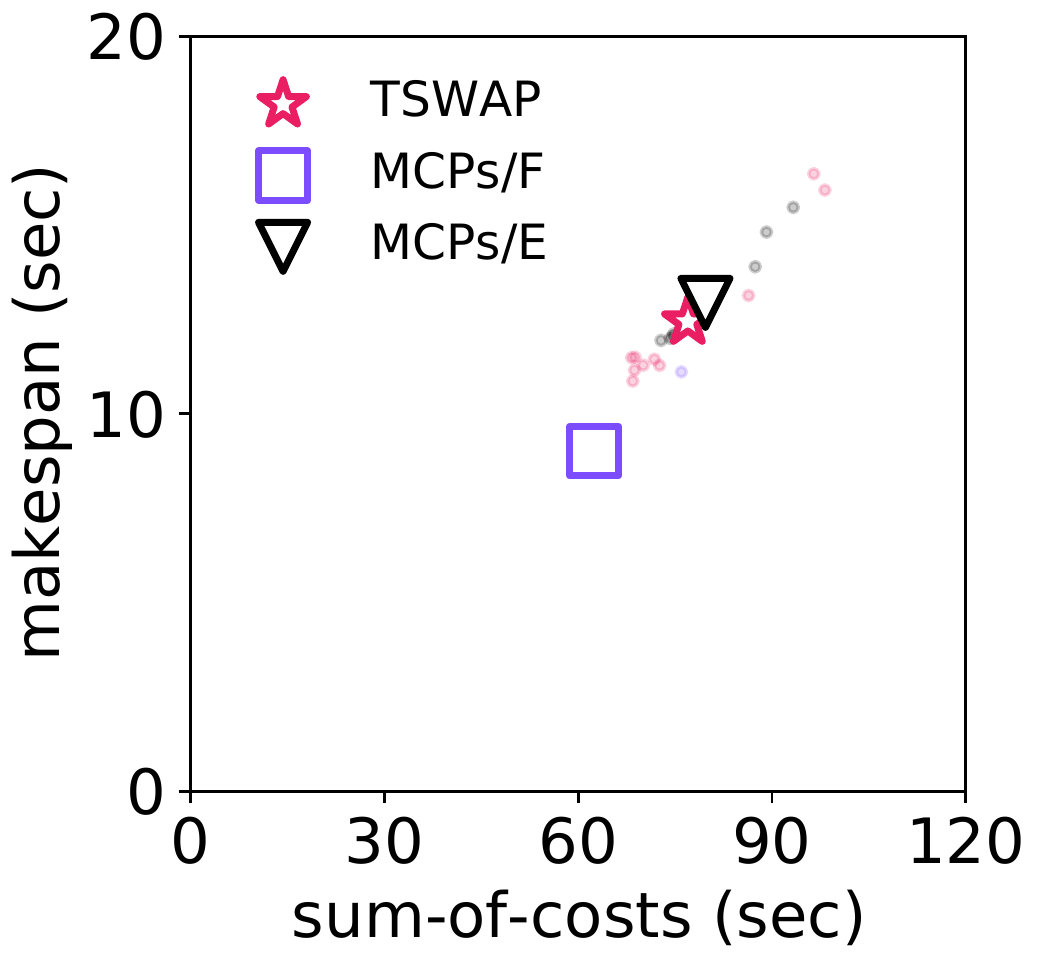}
      \end{minipage}
    \end{tabular}
    \vspace{-0.3cm}
    \caption{
      \textbf{Robot demos with average scores.}
      We also show scatter plots of 10 repetitions by transparent points.
      \emph{left: Delay tolerance.}
      The initial configuration and the offline plan, accompanied by temporal orders, used in MCPs are illustrated.
      Targets are shown by the tips of arrows.
      We made the red one's speed half in one scenario (with delay).
      \emph{middle \& right: Eight robots execution.}
      The initial configurations and virtual grids are shown.
      Targets are marked by white-filled circles.
      ``MCPs/F'' represents MCPs with an offline plan obtained by the makespan-optimal algorithm.
      ``MCPs/E'' represents those with ECBS-TA.
      The sub-optimality of ECBS-TA was $1.12$ and $1.2$ respectively; obtained incrementally increasing values from $1.0$ until solved.
    }
    \label{fig:demo}
  \end{figure*}
}

\begin{theorem}
  Online TSWAP (Algorithm.~\ref{algo:online}) is \emph{complete} for the online time-independent problem.
  \label{thrm:online}
\end{theorem}
\begin{proof}
  The most part is the same as for the offline version (Theorem~\ref{thrm:offline}).
  The problem assumes a fair execution schedule, therefore, $\phi$ must decrease within the sufficiently long period during which all agents are activated at least once; otherwise, $\phi=0$.
\end{proof}
\begin{proposition}
  Regardless of execution schedules, Online TSWAP has upper bounds of;
  \begin{itemize}
  \item maximum-moves: $O(A\cdot\diam(G))$
  \item sum-of-moves: $O(A\cdot\diam(G))$
  \end{itemize}
  \label{prop:online:uppper-bound}
\end{proposition}

\section{Demos of Online Planning}
\label{sec:demo}
This section rates Online TSWAP (using Alg.~\ref{algo:target-assignment}$^\dagger$) with real robots.
The video is available at \url{https://kei18.github.io/tswap}.

\paragraph{Platform}
We used \emph{toio} robots (\url{https://toio.io/}) to implement TSWAP.
The robots, connected to a computer via Bluetooth, evolve on a specific playmat and are controllable by instructions of absolute coordinates.

\paragraph{Usage}
The robots were controlled in a \emph{centralized} style, described as follows.
We created a virtual grid on the playmat; the robots followed the grid.
A central server (a laptop) managed the locations of all robots and issued the instructions (i.e., where to go) to each robot step-by-step.
The instructions were periodically issued to each robot (per \SI{50}{\milli\second}) but they were issued \emph{asynchronously} between robots.
The code was written in Node.js.

\paragraph{Demo of Time Independence}
We first show a three-robots demo highlighting the time independence of TSWAP.
In this demo, the experimenter disturbed robots' progression during the execution (Fig.~\ref{fig:time-independence}).
No matter when the robots move and no matter what order the robots move, the online problem is certainly solved.

\paragraph{Demo of Delay Tolerance}
We next show a simple setting that highlights the delay tolerance of TSWAP.
The comparison baseline is MCPs~\citep{ma2017multi}.
We prepared two scenarios with two robots while manipulating robots' speed.
In one scenario, robots move at usual speeds.
In another scenario, one of them halves its speed, i.e., with delays, assuming an accident happened.

Figure~\ref{fig:demo} shows the configuration and both makespan and sum-of-costs results over 10~repetitions measured in actual time.
Without delays, TSWAP and MCPs do not differ in the results;
however, with delays, the results of TSWAP are clearly better for both metrics.
In MCPs, disturbing/delaying one robot may critically affect the entire performance because it cannot change the temporal orders during the execution, while TSWAP can adaptively address such effects.

\paragraph{Demo with Eight Robots}
We finally present eight robots demos in Fig~\ref{fig:demo}.
Two scenarios were carefully designed to clarify the characteristics of TSWAP.
MCPs' offline plans were obtained by the makespan-optimal algorithm~\cite{yu2013multi} and ECBS-TA.
In the first scenario, TSWAP performed better than the other two, but not so in the second scenario.
This is because the latter has bottleneck nodes that all shortest paths from starts to targets use (two middle nodes in the second column).
Since TSWAP makes robots move following the shortest paths, all robots must use the bottleneck nodes, causing unavoidable congestion.
In contrast, the former does not have such bottleneck nodes, resulting in a small execution time compared to the others.

\section{Conclusion}
\label{sec:conclusion}
This paper presented a novel algorithm TSWAP to solve or execute unlabeled-MAPF; simultaneous target assignment and path planning problems for indistinguishable agents.
TSWAP is complete for both offline and online problems regardless of initial assignments.
We empirically demonstrated that it can solve large offline instances with acceptable solution quality in a very short time, depending on assignment algorithms.
It can also work in online situations with timing uncertainties as shown in our robot demos.

Future directions are:
(1)~applying TSWAP to other situations, e.g., lifelong scenarios, and
(2)~decentralized execution with only local interactions by Online TSWAP.
\section*{Acknowledgments}
We thank the anonymous reviewers for their many insightful comments.
This work was partly supported by JSPS KAKENHI Grant Numbers~20J23011.
Keisuke Okumura thanks the support of the Yoshida Scholarship Foundation.
\bibliography{ref}

\clearpage
\appendix
\section*{Appendix}

\renewcommand{\thesection}{\Alph{section}}
\setcounter{section}{0}
\setcounter{theorem}{1}
\setcounter{proposition}{0}

\section{Proofs}

\begin{proposition}
  Offline TSWAP has upper bounds of;
  \begin{itemize}
  \item makespan: $O(A\cdot\diam(G))$
  \item sum-of-costs: $O(A^2\cdot\diam(G))$
  \item maximum-moves: $O(A\cdot\diam(G))$
  \item sum-of-moves: $O(A\cdot\diam(G))$
  \end{itemize}
\end{proposition}
\begin{proof}
  The potential function $\phi$ in the proof of Theorem~\ref{thrm:offline} is $O(A\cdot\mathit{diam}(G))$; this is the makespan upper bound.
  Note that the second term of $\phi$ is bounded by $\diam(G)$ because $\Pi$ assumes the shortest paths on $G$.
  That of sum-of-costs is trivially obtained by multiplying $|A|$.

  To derive the upper bound of sum-of-moves, consider another potential function $\psi \defeq \sum_{a \in A}\dist{a.v}{a.g}$ in Alg.~\ref{algo:offline}.
  Similarly to $\phi$, $\psi$ is non-increasing.
  $\psi$ becomes zero when the problem is solved.
  $\psi$ is decremented when $a.v$ is updated by $u$ not equal to $a.v$ by Line~\ref{algo:offline:move}, i.e., each ``move'' action decrements $\psi$.
  As a result, $\psi$ eventually reaches zero.
  Since $\psi = O(A\cdot\diam(G))$, we derive the upper bound of sum-of-moves.
  This bound works also for maximum-moves.
\end{proof}

\begin{proposition}
  Assume that the time complexity of $\mathsf{nextNode}$ and the deadlock resolution [Lines~\ref{algo:offline:deadlock-detection}--\ref{algo:offline:rotation}] in Alg.~\ref{algo:offline} are $\alpha$ and $\beta$\footnote{
    We treat $\beta$ as blackbox, but it is obviously implemented by $O(A\cdot\alpha)$.
    Note that this is too conservative in practice, i.e., deadlocks with many agents are rare.
  }, respectively.
  The time complexity of Offline TSWAP excluding Line~\ref{algo:offline:assign} is $O(A^2\cdot\diam(G)\cdot(\alpha + \beta))$.
\end{proposition}
\begin{proof}
  According to Proposition~\ref{prop:offline:upper-bound}, the makespan is $O(A\cdot\diam(G))$, i.e., the repetition number of Line~\ref{algo:offline:loop-start}--\ref{algo:offline:loop-end}.
  Each operation in Lines~\ref{algo:offline:start-for}--\ref{algo:offline:end-for} is constant except for Line~\ref{algo:offline:next}, $O(\alpha)$, and Lines~\ref{algo:offline:deadlock-detection}--\ref{algo:offline:rotation}, $O(\beta)$.
  Those operations repeat exactly $|A|$ times for each timestep, thus deriving the statement.
\end{proof}

\begin{proposition}
  The time complexity of Algorithm~\ref{algo:target-assignment}$^{(\dagger)}$ is $O(\max\left(A(V+E), A^4\right))$.
\end{proposition}
\begin{proof}
  Consider the worst case, i.e., all start-target pairs are evaluated and contained in $\mathcal{B}$.
  The number of vertices and edges of $\mathcal{B}$ are $2|A|$ and $|A|^2$, respectively.
  Then, Line~\ref{algo:ta:mincost-matching} is $O(A^3)$ by the successive shortest path algorithm because its time complexity is $O(f(E^\prime + V^\prime\log{V^\prime}))$ where $f$ is the maximum flow size and $V^\prime$ and $E^\prime$ represent the network.
  The total operations of $\mathit{dist}$ becomes running the breadth-first search $|A|$ times, therefore, $O(A(V+E))$.
  Operations for a priority queue $\mathcal{Q}$ are both $O(\lg{n})$ for extracting and inserting, where $n$ is the length of the queue.
  Thus, the runtime of Line~\ref{algo:ta:setup} is $O(A^2\lg{A})$.
  The queue operations in Lines~\ref{algo:ta:start-loop}--\ref{algo:ta:end-bap} require $O(A^2\lg{A})$.
  Line~\ref{algo:ta:update-matching} finds a single augmenting path and this is linear for the number of edges in $\mathcal{B}$, thus, its complexity is $1+2+\cdots+|A^2| = O(A^4)$.
  As the result, the complexity of Alg.~\ref{algo:target-assignment} is;
  \begin{align*}
    &O(A(V+E)) &\text{finding shortest path}\\
    + &O(A^3) &\text{min-cost maximum matching}\\
    + &O(A^2\lg{A}) &\text{queue operations}\\
    + &O(A^4) &\text{update matching}
  \end{align*}
\end{proof}

\begin{proposition}
  Regardless of execution schedules, Online TSWAP has upper bounds of;
  \begin{itemize}
  \item maximum-moves: $O(|A|\cdot\diam(G))$
  \item sum-of-moves: $O(|A|\cdot\diam(G))$
  \end{itemize}
\end{proposition}
\begin{proof}
  The same proof of Proposition~\ref{prop:offline:upper-bound} is applied.
\end{proof}

\begin{proposition}
  The time complexity of Algorithm~\ref{algo:target-greedy} is $O(A(V+E))$.
\end{proposition}
\begin{proof}
  In the worst case, the algorithm requires $O(A(V+E))$ in total to evaluate all distances of start-target pairs by running the breadth-first search $|A|$ times.
  The operations in Lines~\ref{algo:greedy:obtain-g}--\ref{algo:greedy:fin-operation} repeat at most $|A|^2$ times because each target for each agent is evaluated at most once.
  Lines~\ref{algo:greedy:start-refine}--\ref{algo:greedy:end-refine} repeat at most $\diam(G)$ times because each iteration must reduce the maximum cost of the assignment $\mathcal{M}$.
  Both Line~\ref{algo:greedy:argmax} and Lines~\ref{algo:greedy:start-try-swap}--\ref{algo:greedy:fin-try-swap} are $O(A)$.
  As a result, the algorithm is $O(A(V+E) + A^2 + A\cdot\diam(G))$, which equals to $O(A(V+E))$.
\end{proof}

We additionally present the correctness of Alg.~\ref{algo:target-greedy};
\begin{proposition}
  Algorithm~\ref{algo:target-greedy} is correct; returns a distinct target for each agent.
  \label{prop:correctness-greedy}
\end{proposition}
\begin{proof}
  We focus on the initial assignment phase [Lines~\ref{algo:greedy:init}--\ref{algo:greedy:fin-init-assign}] because the refinement phase [Lines~\ref{algo:greedy:start-refine}--\ref{algo:greedy:end-refine}] only swaps targets for the existing assignment $\mathcal{M}$ and terminates within finite iterations (see Prop.~\ref{prop:complexity-greedy}).
  Trivially, each assignment operation never assigns one target to more than one agents.
  Observe that $|\mathcal{U}| + |\mathcal{M}| = |A|$ is invariant.
  Thus, an output of the algorithm is correct.

  The termination of the algorithm is derived by contradiction.
  Assume the invalid state of the algorithm; Line~\ref{algo:greedy:obtain-g} does not have any corresponding values, meaning that all targets have already been evaluated.
  This violates the invariance of $|\mathcal{U}| + |\mathcal{M}| = |A|$; hence such states never be realized.
  For each agent, no target is evaluated more than once.
  Therefore, the algorithm eventually terminates.
\end{proof}

\begin{table*}[t!]
  \centering
  {
    \newcommand{\cmid}{\cmidrule(lr){1-18}}
    \newcommand{\w}[1]{\textbf{#1}}  
    \renewcommand{\c}[1]{{\tiny(#1)}}
    \setlength{\tabcolsep}{2pt}
    \scriptsize
    {\normalsize\textit{random-64-64-20}}\\\smallskip
    \begin{tabular}{rrrlrlrlrlrlrlrlrl}
      \toprule
      $|A|$
      & metric
      & \multicolumn{2}{c}{Alg.~\ref{algo:target-assignment}}
      & \multicolumn{2}{c}{Alg.~\ref{algo:target-assignment}$^\dagger$}
      & \multicolumn{2}{c}{Alg.~\ref{algo:target-assignment}$^{\dagger\ast}$}
      & \multicolumn{2}{c}{Alg.~\ref{algo:target-greedy}}
      & \multicolumn{2}{c}{Alg.~\ref{algo:target-greedy}$^\ast$}
      & \multicolumn{2}{c}{Alg.~\ref{algo:target-greedy-soc}}
      & \multicolumn{2}{c}{greedy}
      & \multicolumn{2}{c}{linear}
      \\ \midrule
      \multirow{3}{*}{110}
      & runtime (ms) & 6 & \c{6,7} & 10 & \c{10,11} & 17 & \c{17,17} & \w{4} & \c{4,4} & 12 & \c{12,12} & \w{4} & \c{4,4} & 12 & \c{12,12} & 23 & \c{22,23} \\
      & makespan & \w{17} & \c{17,18} & \w{17} & \c{17,18} & \w{17} & \c{17,18} & 20 & \c{20,21} & 20 & \c{20,21} & 36 & \c{33,39} & 84& \c{79,88} & 36& \c{34,39} \\
      & sum-of-costs & 1079 & \c{1036,1119} & \w{937} & \c{890,972} & \w{937} & \c{890,972} & 1139 & \c{1094,1184} & 1136 & \c{1091,1180} & \w{958} & \c{919,995} & 1396 & \c{1320,1470} & \w{940} & \c{900,980} \\ \cmid
      \multirow{3}{*}{500}
      & runtime (ms) & 72 & \c{64,79} & 174 & \c{164,184} & 213 & \c{203,222} & \w{15} & \c{15,16} & 78 & \c{77,78} & 22 & \c{21,23} & 77 & \c{77,78} & 602 & \c{595,610}\\
      & makespan & \w{10} & \c{10,11} & \w{11} & \c{10,11} & \w{11} & \c{10,11} & 13 & \c{12,13} & 13 & \c{12,13} & 34 & \c{31,36} & 79 & \c{75,82} & 32 & \c{29,35} \\
      & sum-of-costs & 2595 & \c{2497,2690} & \w{2169} & \c{2084,2252} & \w{2169} & \c{2083,2251} & 2878 & \c{2761,2993} & 2860 & \c{2749,2969} & 2546 & \c{2432,2653} & 4653 & \c{4394,4908} & 2429 & \c{2306,2550} \\ \cmid
      \multirow{3}{*}{1000}
      & runtime (ms) & 335 & \c{281,383} & 757 & \c{692,813} & 882 & \c{814,941} & \w{28} & \c{26,29} & 233 & \c{232,234} & 58 & \c{56,60} & 221 & \c{220,222} & 3811 & \c{3801,3120}\\
      & makespan & \w{9} & \c{8,9} & \w{9} & \c{9,9} & \w{9} & \c{9,9} & 11 & \c{10,11} & 11 & \c{10,11} & 28 & \c{26,31} & 71 & \c{67,74} & 26 & \c{23,28} \\
      & sum-of-costs & 3591 & \c{3448,3728} & \w{2922} & \c{2812,3030} & \w{2922} & \c{2811,3032} & 4020 & \c{3855,4179} & 4043 & \c{3874,4212} & 3695 & \c{3533,3857} & 7571 & \c{7112,8011} & 3491 & \c{3306,3669} \\ \cmid
      \multirow{3}{*}{2000}
      & runtime (ms) & 1453 & \c{1258,1627} & 3035 & \c{2814,3233} & 3205 & \c{2994,3402} & \w{66} & \c{62,70} & 784 & \c{781,787} & 188 & \c{182,194} & 725 & \c{722,727} & 32944 & \c{32805,33082}\\
      & makespan & \w{8} & \c{7,8} & \w{7} & \c{7,7} & \w{7} & \c{7,7} & 10 & \c{10,11} & 10 & \c{10,11} & 24 & \c{22,26} & 56 & \c{53,59} & 23 & \c{21,25} \\
      & sum-of-costs & 4670 & \c{4471,4862} & \w{3469} & \c{3311,3620} & \w{3469} & \c{3310,3622} & 5200 & \c{4955,5435} & 5244 & \c{5001,5471} & 5465 & \c{5112,5796} & 12292 & \c{11357,13145} & 5122 & \c{4750,5479} \\ \bottomrule
    \end{tabular}\\\bigskip
    \setlength{\tabcolsep}{0.9pt}
    {\normalsize\textit{lak303d}}\\\smallskip
    \begin{tabular}{rrrlrlrlrlrlrlrlrl}
      \toprule
      $|A|$
      & metric
      & \multicolumn{2}{c}{Alg.~\ref{algo:target-assignment}}
      & \multicolumn{2}{c}{Alg.~\ref{algo:target-assignment}$^\dagger$}
      & \multicolumn{2}{c}{Alg.~\ref{algo:target-assignment}$^{\dagger\ast}$}
      & \multicolumn{2}{c}{Alg.~\ref{algo:target-greedy}}
      & \multicolumn{2}{c}{Alg.~\ref{algo:target-greedy}$^\ast$}
      & \multicolumn{2}{c}{Alg.~\ref{algo:target-greedy-soc}}
      & \multicolumn{2}{c}{greedy}
      & \multicolumn{2}{c}{linear}
      \\ \midrule
      \multirow{3}{*}{100}
      & runtime (ms) & 42 & \c{39,44} & 49 & \c{46,51} & 51 & \c{50,53} & \w{28} & \c{26,30} & 46 & \c{44,47} & 35 & \c{31,39} & 56 & \c{54,59} & 56 & \c{53,58} \\
      & makespan & \w{89} & \c{83,94} & \w{89} & \c{83,94} & \w{89} & \c{83,94} & \w{89} & \c{83,95} & \w{89} & \c{83,95} & 229 & \c{204,253} & 387 & \c{363,412} & 242 & \c{217,267} \\
      & sum-of-costs & 3834 & \c{3540,4108} & \w{3264} & \c{3034,3483} & \w{3264} & \c{3032,3483} & 3934 & \c{3650,4204} & 3930 & \c{3635,4205} & \w{3460} & \c{3184,3712} & 4561 & \c{4250,4862} & \w{3511} & \c{3239,3760} \\ \cmid
      \multirow{3}{*}{500}
      & runtime (ms) & 604 & \c{503,690} & 855 & \c{741,958} & 863 & \c{760,954} & \w{105} & \c{98,112} & 234 & \c{230,238} & 158 & \c{145,171} & 276 & \c{267,285} & 758 & \c{748,767} \\
      & makespan & \w{60} & \c{56,65} & \w{61} & \c{56,66} & \w{61} & \c{56,66} & \w{61} & \c{56,66} & \w{61} & \c{56,66} & 264 & \c{236,291} & 425 & \c{404,448} & 287 & \c{257,316} \\
      & sum-of-costs & 11148 & \c{10122,12060} & \w{9236} & \c{8410,9964} & \w{9236} & \c{8439,9972} & 12488 & \c{11404,13479} & 12517 & \c{11435,13552} & 11594 & \c{10290,12819} & 17989 & \c{16654,19273} & 12228 & \c{10903,13498} \\ \cmid
      \multirow{3}{*}{1000}
      & runtime (ms) & 3406 & \c{2894,3886} & 4529 & \c{3956,5081} & 4618 & \c{4051,5142} & \w{174} & \c{163,184} & 546 & \c{539,552} & 286 & \c{266,306} & 570 & \c{560,580} & 4081 & \c{4064,4097} \\
      & makespan & \w{46} & \c{43,49} & \w{47} & \c{44,50} & \w{47} & \c{44,50} & \w{46} & \c{43,50} & \w{47} & \c{43,50} & 215 & \c{192,237} & 395 & \c{375,416} & 240 & \c{215,264} \\
      & sum-of-costs & 14353 & \c{13274,15323} & \w{11909} & \c{11059,12690} & \w{11909} & \c{11051,12708} & 16954 & \c{15645,18091} & 17009 & \c{15727,18136} & 15492 & \c{14003,16864} & 28783 & \c{26794,30555} & 16355 & \c{14777,17839} \\ \cmid
      \multirow{3}{*}{2000}
      & runtime (ms) & 25402 & \c{20525,29840} & 31470 & \c{26244,36250} & 31630 & \c{26506,36466} & \w{393} & \c{359,424} & 1536 & \c{1508,1560} & 624 & \c{589,658} & 1400 & \c{1384,1415} & 33813 & \c{33724,33902} \\
      & makespan & \w{47} & \c{42,51} & \w{48} & \c{43,53} & \w{48} & \c{43,53} & \w{48} & \c{43,52} & \w{47} & \c{43,52} & 184 & \c{164,204} & 373 & \c{355,391} & 212 & \c{189,235} \\
      & sum-of-costs & 21200 & \c{19523,22694} & \w{18176} & \c{16783,19498} & \w{18176} & \c{16759,19494} & 26806 & \c{24760,28786} & 26706 & \c{24611,28660} & 24127 & \c{21761,26297} & 50388 & \c{47276,53360} & 25961 & \c{23366,28322} \\ \bottomrule
    \end{tabular}

  }
  \caption{\textbf{The detailed results of TSWAP with different assignment algorithms.}
    We also display 95\% confidence intervals of the mean, on which bold characters are based.
  }
  \label{table:result-assignment-details}
\end{table*}

\section{Refinement for Sum-of-costs}
The sum-of-costs version of Alg.~\ref{algo:target-greedy} is presented in Alg.~\ref{algo:target-greedy-soc}.

\begin{algorithm}[tb]
  \caption{\textbf{Refinement for Sum-of-costs}}
  \label{algo:target-greedy-soc}
  {\small
  \begin{algorithmic}[1]
    \item[\textbf{input}:~unlabeled-MAPF instance]
    \item[\textbf{output}:~$\mathcal{M}$: assignment, a set of pairs $s \in \mathcal{S}$ and $g \in \mathcal{T}$]
      \STATE execute Lines~\ref{algo:greedy:init}--\ref{algo:greedy:fin-init-assign} of Alg.~\ref{algo:target-greedy}
      \smallskip
      \WHILE{$\mathcal{M}$ is updated in the last iteration}
      \label{algo:greedy-soc:while}
      \FOR{$(s_i, s_i)\in\mathcal{M}, (s_j, g_j)\in\mathcal{M}, i\neq j$}
      \STATE $c_{\text{now}} \leftarrow \dist{s_i}{g_i} + \dist{s_j}{g_j}$
      \IFSINGLE{$h(s_j, g_i) + h(s_i, g_j) \geq c_{\text{now}}$}{\textbf{continue}}
      \STATE $c_{\text{swap}} \leftarrow \dist{s_j}{g_i} + \dist{s_i}{g_j}$
      \IFSINGLE{$c_{\text{swap}} < c_{\text{now}}$}{swap $g_i$ and $g_j$ of $\mathcal{M}$; \textbf{break}}
      \ENDFOR
      \ENDWHILE
      \label{algo:greedy-soc:end-while}
  \end{algorithmic}
  }
\end{algorithm}

\begin{proposition}
  The time complexity of Algorithm~\ref{algo:target-greedy-soc} is $O(A(V+E) + A^3\cdot\diam(G))$.
\end{proposition}
\begin{proof}
  Lines~\ref{algo:greedy-soc:while}--\ref{algo:greedy-soc:end-while} repeat at most $|A|\cdot\diam(G)$.
  Each iteration requires $O(A^2)$.
  Together with the proof of Proposition~\ref{prop:complexity-greedy}, we derive the statement.
\end{proof}

\section{Further Evaluation of Effect of Initial Target Assignment}

Table~\ref{table:result-assignment-details} presents further details of Table~\ref{table:result-assignment} and an additional result of another map $\textit{lak303d}$.
The main observations are the same as described in Sec.~\ref{sec:eval:assign}.

\section{Implementation of Offline TSWAP}

We use a priority queue with re-insert operations for agents instead of a simple list for Line~\ref{algo:offline:start-for}--\ref{algo:offline:end-for} because the ordering of a list affects results.
To observe this, consider a line graph with two adjacent agents on the left side.
Their targets are on the right side.
With a simple list implementation, when the left agent plans prior to the right one, the left agent has to wait for one timestep until the right agent has moved.
In the reverse case, this wait action never happens.
Therefore, we avoid such wasteful wait actions by using the priority queue with re-insert operations.

\section{Implementation of the Polynomial-Time Makespan-Optimal Algorithm}
In our experiment, we used the polynomial-time makespan-optimal algorithm~\cite{yu2013multi}.
This algorithm has several techniques to improve the runtime performance.
All of them are straightforward, however, \emph{their quantitative evaluation has not been performed to our knowledge}.
Thus, we evaluated them and selected the best one for each experimental setting.
This section describes the details.

\subsection{Preliminaries --- Algorithm Description}
\label{subsec:flow-algo}
Given a timestep $T$, a decision problem of whether an unlabeled-MAPF instance has a solution with makespan $T$ can be solved in polynomial time.
This is achieved by a reduction to maximum flow problems on a large graph called \emph{time expanded network}~\cite{yu2013multi}.
\footnote{
  We slightly change the structure of the network in the original paper to make the network slim, i.e., removing internal two vertices for preventing swap conflicts. In the unlabeled setting, plans with swap conflicts can be easily converted to plans without conflicts. This technique is used in~\cite{ma2016multi,liu2019task}.
}
Let denote \network{T} be the time expanded network for makespan $T$.
To clarify the context, we use ``vertices'' for the network \network{T} and ``nodes'' for the original graph $G$.

For each timestep $0 \leq t < T$ and each node $v \in V$, the network \network{T} has two vertices $v^{t}_\text{in}$ and $v^{t}_\text{out}$.
In addition, there are two special vertices \emph{source} and \emph{sink} to convert the unlabeled-MAPF instance to the maximum flow problem.
\network{T} has five types of edges with a unit capacity.
The intuitions are the following.
\begin{itemize}
  \setlength{\itemsep}{0cm}
\item $(v^t_\text{in}, v^t_\text{out})$: An agent can stay at $v$ during $[t, t+1]$.
\item $(u^t_\text{in}, v^t_\text{out})$ if $(u, v) \in E$: An agent can move from $u$ to $v$ during $[t, t+1]$.
\item $(v^t_\text{out}, v^{t+1}_\text{in})$: Prevent vertex conflicts.
\item $(\text{source}, v^0_\text{in})$ if $v \in \mathcal{S}$: Initial locations.
\item $(v^{T-1}_\text{out}, \text{sink})$ if $v \in \mathcal{G}$: Targets.
\end{itemize}
We show an example of time expanded networks in Fig.~\ref{fig:time-expanded-network} with the maximum flows.
Once the maximum flow with size equals to $|A|$ is obtained, the solution for the unlabeled-MAPF instance is easily obtained from the flow.

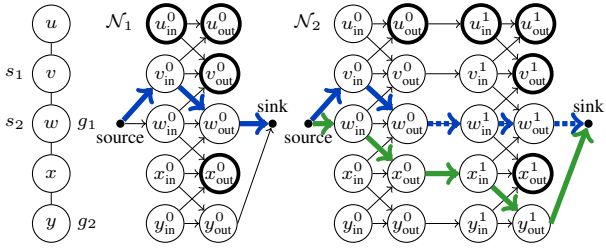
\begin{figure}[t]
  \centering
  \begin{tikzpicture}
    \newcommand{\betweenvertex}{0.15 cm}
    \newcommand{\rightmarginfrominvertex}{0.45 cm}
    \newcommand{\rightmarginfromoutvertex}{0.7 cm}
    \scriptsize
    \node[vertex](v1) at (0, 0) {$y$};
    \node[vertex,above=\betweenvertex of v1](v2) {$x$};
    \node[vertex,above=\betweenvertex of v2](v3) {$w$};
    \node[vertex,above=\betweenvertex of v3](v4) {$v$};
    \node[vertex,above=\betweenvertex of v4](v5) {$u$};
    \foreach \u / \v in {v1/v2, v2/v3, v3/v4, v4/v5} \draw[edge] (\u) -- (\v);
    \node[left=0cm of v4]() {$s_1$};
    \node[left=0cm of v3]() {$s_2$};
    \node[right=0cm of v1]() {$g_2$};
    \node[right=0cm of v3]() {$g_1$};
    %
    {
      \newcommand{\rightmarginfromoriginal}{1.0 cm}
      \node[vertex-small,label=below:{source},right=0.6cm of v3](source) {};
      \node[above=1.1cm of source]() {\network{1}};
      \node[vertex-large,right=\rightmarginfromoriginal of v1](v1_0_in) {$y^0_{\text{in}}$};
      \node[vertex-large,right=\rightmarginfromoriginal of v2](v2_0_in) {$x^0_{\text{in}}$};
      \node[vertex-large,right=\rightmarginfromoriginal of v3](v3_0_in) {$w^0_{\text{in}}$};
      \node[vertex-large,right=\rightmarginfromoriginal of v4](v4_0_in) {$v^0_{\text{in}}$};
      \node[vertex-large-dead,right=\rightmarginfromoriginal of v5](v5_0_in) {$u^0_{\text{in}}$};
      \node[vertex-large,right=\rightmarginfrominvertex of v1_0_in.center](v1_0_out) {$y^0_{\text{out}}$};
      \node[vertex-large-dead,right=\rightmarginfrominvertex of v2_0_in.center](v2_0_out) {$x^0_{\text{out}}$};
      \node[vertex-large,right=\rightmarginfrominvertex of v3_0_in.center](v3_0_out) {$w^0_{\text{out}}$};
      \node[vertex-large-dead,right=\rightmarginfrominvertex of v4_0_in.center](v4_0_out) {$v^0_{\text{out}}$};
      \node[vertex-large-dead,right=\rightmarginfrominvertex of v5_0_in.center](v5_0_out) {$u^0_{\text{out}}$};
      \node[vertex-small,label=above:{sink},right=1.9 cm of source](sink) {};
      \foreach \u / \v in {source/v4_0_in,source/v3_0_in,
        v1_0_in/v1_0_out,v2_0_in/v2_0_out,v3_0_in/v3_0_out,v4_0_in/v4_0_out,v5_0_in/v5_0_out,
        v1_0_in/v2_0_out,v2_0_in/v1_0_out,v2_0_in/v3_0_out,v3_0_in/v2_0_out,v3_0_in/v4_0_out,v4_0_in/v3_0_out,v4_0_in/v5_0_out,v5_0_in/v4_0_out,
        v1_0_out.east/sink,v3_0_out/sink}
      \draw[edge,->] (\u) -- (\v);
      %
      \foreach \u / \v in {source/v4_0_in,v4_0_in/v3_0_out,v3_0_out/sink}
      \draw[flow,color={rgb:red,0;green,1;blue,3}] (\u) -- (\v);
    }
    %
    {
      \newcommand{\rightmarginfromoriginal}{3.5 cm}
      \node[vertex-small,label=below:{source},right=3.1cm of v3](source){};
      \node[above=1.1cm of source]() {\network{2}};
      \node[vertex-large,right=\rightmarginfromoriginal of v1](v1_0_in) {$y^0_{\text{in}}$};
      \node[vertex-large,right=\rightmarginfromoriginal of v2](v2_0_in) {$x^0_{\text{in}}$};
      \node[vertex-large,right=\rightmarginfromoriginal of v3](v3_0_in) {$w^0_{\text{in}}$};
      \node[vertex-large,right=\rightmarginfromoriginal of v4](v4_0_in) {$v^0_{\text{in}}$};
      \node[vertex-large,right=\rightmarginfromoriginal of v5](v5_0_in) {$u^0_{\text{in}}$};
      \node[vertex-large,right=\rightmarginfrominvertex of v1_0_in.center](v1_0_out) {$y^0_{\text{out}}$};
      \node[vertex-large,right=\rightmarginfrominvertex of v2_0_in.center](v2_0_out) {$x^0_{\text{out}}$};
      \node[vertex-large,right=\rightmarginfrominvertex of v3_0_in.center](v3_0_out) {$w^0_{\text{out}}$};
      \node[vertex-large,right=\rightmarginfrominvertex of v4_0_in.center](v4_0_out) {$v^0_{\text{out}}$};
      \node[vertex-large-dead,right=\rightmarginfrominvertex of v5_0_in.center](v5_0_out) {$u^0_{\text{out}}$};
      \node[vertex-large,right=\rightmarginfromoutvertex of v1_0_out.center](v1_1_in) {$y^1_{\text{in}}$};
      \node[vertex-large,right=\rightmarginfromoutvertex of v2_0_out.center](v2_1_in) {$x^1_{\text{in}}$};
      \node[vertex-large,right=\rightmarginfromoutvertex of v3_0_out.center](v3_1_in) {$w^1_{\text{in}}$};
      \node[vertex-large,right=\rightmarginfromoutvertex of v4_0_out.center](v4_1_in) {$v^1_{\text{in}}$};
      \node[vertex-large-dead,right=\rightmarginfromoutvertex of v5_0_out.center](v5_1_in) {$u^1_{\text{in}}$};
      \node[vertex-large,right=\rightmarginfrominvertex of v1_1_in.center](v1_1_out) {$y^1_{\text{out}}$};
      \node[vertex-large-dead,right=\rightmarginfrominvertex of v2_1_in.center](v2_1_out) {$x^1_{\text{out}}$};
      \node[vertex-large,right=\rightmarginfrominvertex of v3_1_in.center](v3_1_out) {$w^1_{\text{out}}$};
      \node[vertex-large-dead,right=\rightmarginfrominvertex of v4_1_in.center](v4_1_out) {$v^1_{\text{out}}$};
      \node[vertex-large-dead,right=\rightmarginfrominvertex of v5_1_in.center](v5_1_out) {$u^1_{\text{out}}$};
      \node[vertex-small,label=above:{sink},right=3.6 cm of source](sink) {};
      \foreach \u / \v in {source/v4_0_in,source/v3_0_in,
        v1_0_in/v1_0_out,v2_0_in/v2_0_out,v3_0_in/v3_0_out,v4_0_in/v4_0_out,v5_0_in/v5_0_out,
        v1_0_out/v1_1_in,v2_0_out/v2_1_in,v3_0_out/v3_1_in,v4_0_out/v4_1_in,v5_0_out/v5_1_in,
        v1_1_in/v1_1_out,v2_1_in/v2_1_out,v3_1_in/v3_1_out,v4_1_in/v4_1_out,v5_1_in/v5_1_out,
        v1_0_in/v2_0_out,v2_0_in/v1_0_out,v2_0_in/v3_0_out,v3_0_in/v2_0_out,v3_0_in/v4_0_out,v4_0_in/v3_0_out,v4_0_in/v5_0_out,v5_0_in/v4_0_out,
        v1_1_in/v2_1_out,v2_1_in/v1_1_out,v2_1_in/v3_1_out,v3_1_in/v2_1_out,v3_1_in/v4_1_out,v4_1_in/v3_1_out,v4_1_in/v5_1_out,v5_1_in/v4_1_out}
      \draw[edge,->] (\u) -- (\v);
      %
      \foreach \u / \v in {source/v4_0_in,v4_0_in/v3_0_out}
      \draw[flow,color={rgb:red,0;green,1;blue,3}] (\u) -- (\v);
      \foreach \u / \v in {v3_0_out/v3_1_in,v3_1_in/v3_1_out,v3_1_out/sink}
      \draw[flow,color={rgb:red,0;green,1;blue,3},densely dotted] (\u) -- (\v);
      \foreach \u / \v in {source/v3_0_in,v3_0_in/v2_0_out,v2_0_out/v2_1_in,v2_1_in/v1_1_out,v1_1_out.east/sink}
      \draw[flow,color={rgb:red,1;green,3;blue,1}] (\u) -- (\v);
    }
  \end{tikzpicture}
  \vspace{-0.2cm}
  \caption{
    \textbf{Examples of time expanded network and two techniques (prune and reuse).}
    The left shows an unlabeled-MAPF instance.
    The center shows \network{1} with the maximum flow (blue).
    Since the size of maximum flow is not equal to $|A|$, there is no feasible solution with makespan $T=1$.
    The right shows \network{2} with the maximum flow (blue and green).
    The resulting solution is $\path{1}=(v, w, w)$ and $\path{2}=(w, x, y)$.
    Since vertices with bold lines, e.g., $u_{\text{out}}^0$ in the both networks, never reach the sink, they can be pruned during the search for augmenting paths.
    When extending timestep, the past flow (blue solid line in \network{1}) can be effectively reused to create a new flow (blue dotted line in \network{2}).
  }
  \label{fig:time-expanded-network}
\end{figure}

Since many polynomial-time maximum flow algorithms exist, the maximum flow problem for time expanded networks can be solved in polynomial-time.
For instance, the time complexity of the Ford-Fulkerson algorithm~\cite{ford1956maximal}, a major algorithm for the maximum flow problem, is $O(fE^\prime)$ where $f$ is the maximum flow size and $E^\prime$ denotes edges in the network;
the running time in \network{T} is $O(AVT)$ with a natural assumption of $E=O(V)$.
According to~\cite{yu2013multi}, $T=A+V-2$ in the worst case, thus, the time complexity is $O(AV^2)$.

Using the above scheme, the remaining problem is to find an optimal $T$.
This phase has many design choices.
The typical one is incremental search (i.e., $T=1,2,3,4\ldots$).

\subsection{Techniques}
This part introduces three effective techniques to speedup the optimal algorithm.
We assume that the Fold-Furlkerson algorithm is used to find the maximal flow.
The first two techniques are about finding an optimal makespan $T$.
The last one is for reducing the search effort of the maximum flow;
this is new in MAPF literature.

\paragraph{Lower Bound}
Starting the search for $T$ from makespan lower bound is expected to reduce the computational effort because the number of solving the maximum flow problems is reduced.
A naive approach to obtain the bound is computing $\max$$_{i}\min$$_{j} h(\loc{i}{0}, g_j)$.
\footnote{
  Or, \dist{\loc{i}{0}}{g_j} but we avoid this because in most cases the admissible heuristics work well and it is much faster.
}
A tighter bound is obtained by solving the bottleneck assignment problem~\cite{gross1959bottleneck}, i.e., assigning each agent to one target while minimizing the maximum cost, regarding distances between initial locations and targets as costs.
This bound is easily obtained by an adaptive version of Alg.~\ref{algo:target-assignment}.

\begin{figure*}[t]
  \centering
  \includegraphics[width=.8\hsize]{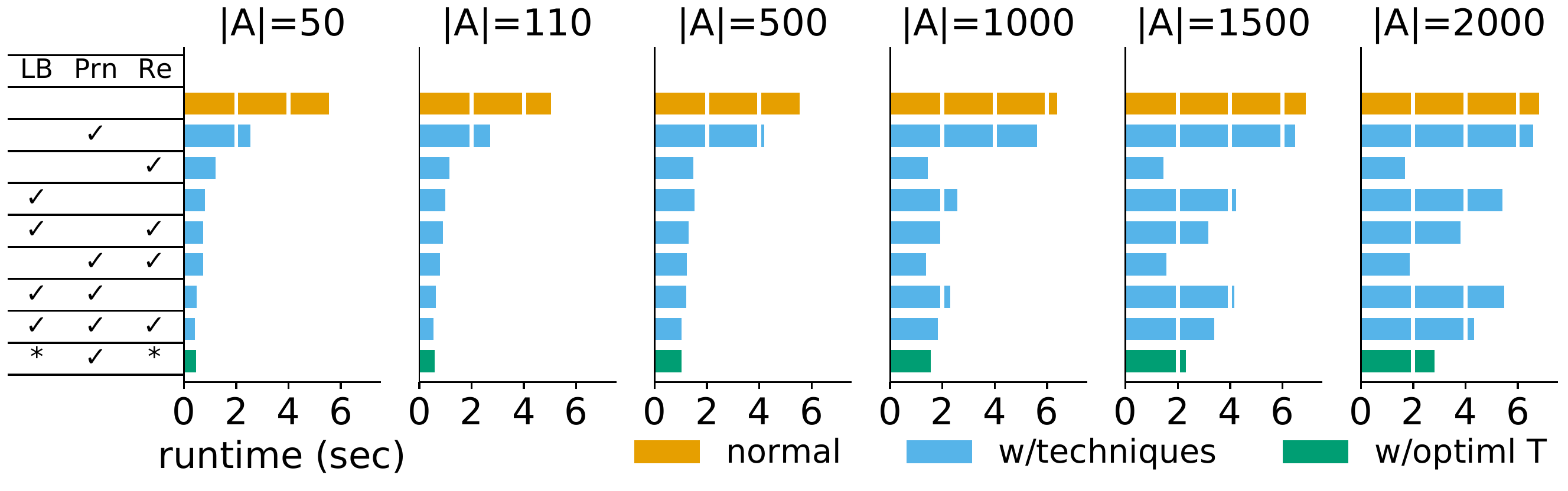}
  \vspace{-0.4cm}
  \caption{\textbf{The average runtime of the optimal algorithm in \textit{random-64-64-20}.}
    Checkmarks at ``LB'' mean starting the search from the lower bound obtained by Alg.~\ref{algo:target-assignment}; otherwise, it is obtained by computing $\max_{i}\left(\min_{j}\left( h(\loc{i}{0}, g_j) \right)\right)$.
    ``Prn'' stands for pruning.
    ``Re'' stands for reusing past flows.
  }
  \label{fig:result-optimal}
\end{figure*}

\paragraph{Pruning of Redundant Vertices}
During the search of augmenting paths, vertices that never reach the sink can be pruned.
We highlight such vertices by bold lines in Fig.~\ref{fig:time-expanded-network}.
The pruning is realized by two processes.

\begin{itemize}
\item Preprocessing: Before searching optimal makespans, calculate the minimum distance to reach one of the targets from each node $v \in V$.
  Let denote this distance $\lambda(v)$, e.g., $\lambda(u)=2$ in Fig.~\ref{fig:time-expanded-network}.
  This is computed by an one-shot breadth-first search from all targets; its time complexity is $O(V+E)$, i.e., the overhead of the preprocessing.
\item Pruning: During the search of augmenting paths, $v^{t}_\text{out}$ such that $t + \lambda(v) \geq T$ is avoided from expanding as successors.
  This also prevents from expanding $v^{t+1}_\text{in}$.
\end{itemize}
Pruning reduces search time of the maximum flow algorithm without affecting its correctness and optimality.
This concept to flow network can be seen in~\citeAppendix{yu2016optimal}, while similar concepts can be seen in other reduction-based approaches to \emph{labeled} MAPF, e.g., SAT-based~\citeAppendix{surynek2016efficient} and ASP-based~\citeAppendix{gomez2020solving}.

\paragraph{Reuse of Past Flows}
Consider the incremental search of optimal makespan and expanding the network from \network{T} to \network{T+1}.
The Ford-Fulkerson algorithm iteratively finds an augmenting path until no such path exists. Thus, a reduction of the iterations is expected to reduce computation time.

A feasible flow of \network{T+1} with size equal to the maximum flow of \network{T} can be obtained immediately without search.
To see this, let $v^{T-1}_\text{out}$ be a vertex used in the maximum flow of \network{T}.
Let this flow extending for \network{T+1} by using $v^{T-1}_\text{out}$, $v^{T}_\text{in}$, $v^{T}_\text{out}$, and the sink.
In Fig.~\ref{fig:time-expanded-network}, we show the example of \network{2} highlighted by a blue dotted line started from $w^{0}_\text{out}$.
This new flow is trivially feasible in \network{2};
in general, it is feasible in \network{T+1}.
As a result, the Ford-Fulkerson algorithm in \network{2} only needs to find one augmenting path (green), rather than two.
Hence, the reuse of the past flow contributes to reducing the iterations of the Ford-Fulkerson algorithm.

\subsection{Evaluation of Techniques}
\label{subsec:eval-optimal}
We evaluated the three techniques using a 4-connected grid \textit{random-64-64-20}, shown in Fig.~\ref{fig:maps}, while changing the number of agents.
The simulator and the experimental environment were the same as Section~\ref{sec:evaluation}.
All instances were created by choosing randomly initial locations and targets.

The average runtime over $50$ instances is shown in Fig.~\ref{fig:result-optimal}.
We additionally show a single run of the maximum flow algorithm with optimal makespan, unknown before experiments (green bars).
Since all combinations yield optimal solutions, the smaller runtime is better.

As for the technique of the lower bounds, we tested two: the conservative one obtained by $\max$$_{i}\min$$_{j} h(\loc{i}{0}, g_j)$ (without checkmarks at ``LB''), or, the aggressive one obtained by solving the bottleneck assignment problem using Alg.~\ref{algo:target-assignment} (with checkmarks).
The runtime includes computing the bounds.
The aggressive one has an advantage when the number of agents is small;
however, as increasing, solving the assignment problem itself takes time then it loses the advantage.
Rather, the conservative one scores smaller runtime.

The other two techniques surely contribute to reducing runtime.
Notably, the best runtimes with the proposed techniques (blue) do not differ or are faster from those given the optimal makespan (green).

\subsection{Implementations in the Experiments}
Following the above result, in our experiments, the optimal algorithm used the techniques of the aggressive ``LB'', ``Prn'', and ``Re'' except for $|A| \geq 1000$; in this case, it used the conservative ``LB'' instead of the aggressive one because Alg.~\ref{algo:target-assignment} becomes costly.
\textit{brc202d} is an exception; we used aggressive ``LB'' even when $|A| \geq 1000$.
Since the map is too large, conservative ``LB'' more often failed to find solutions within a time limit.

\bibliographystyleAppendix{aaai22}
\bibliographyAppendix{ref}

\end{document}